%% file: main.tex
\documentclass[twoside]{article}

\usepackage[preprint]{aistats2026}

\usepackage[round]{natbib}

\renewcommand{\cite}{\citep}

\usepackage{tikz}
\usetikzlibrary{positioning}
\input{notation}

\begin{document}

\runningauthor{Sunmook Choi, Yahya Sattar, Yassir Jedra, Maryam Fazel, Sarah Dean}

\twocolumn[

\aistatstitle{Explore-then-Commit for Nonstationary Linear Bandits with Latent Dynamics}

\aistatsauthor{ 
Sunmook Choi \hspace{-5mm} \And
Yahya Sattar \And 
Yassir Jedra \And 
Maryam Fazel \And 
Sarah Dean }

\aistatsaddress{ 
Cornell \hspace{-5mm} \And 
Cornell  \And 
Imperial College London  \And 
U Washington 
\And Cornell} 
]

\begin{abstract}
    We study a nonstationary bandit problem where rewards depend on both actions and latent states, the latter governed by unknown linear dynamics. Crucially, the state dynamics also depend on the actions, resulting in tension between short-term and long-term rewards. We propose an explore-then-commit algorithm for a finite horizon $T$. During the exploration phase, random Rademacher actions enable estimation of the Markov parameters of the linear dynamics, which characterize the action-reward relationship. In the commit phase, the algorithm uses the estimated parameters to design an optimized action sequence for long-term reward. Our proposed algorithm achieves $\tilde{\mathcal{O}}(T^{2/3})$ regret. Our analysis handles two key challenges: learning from temporally correlated rewards, and designing action sequences with optimal long-term reward. We address the first challenge by providing near-optimal sample complexity and error bounds for system identification using bilinear rewards. We address the second challenge by proving an equivalence with indefinite quadratic optimization over a hypercube, a known NP-hard problem. We provide a sub-optimality guarantee for this problem, enabling our regret upper bound. Lastly, we propose a semidefinite relaxation with Goemans-Williamson rounding as a practical approach.
\end{abstract}

\vspace{-12pt}
\section{Introduction}
Many application domains, like personalized recommendations or online advertising, require sequential decision-making under uncertainty. 
Classical bandit algorithms address the trade-off between reducing uncertainty and optimizing performance in environments where rewards do not depend on the algorithm's past decisions. 
Many real world systems exhibit temporal dependencies -- actions influence not only immediate reward, but also the future state of the environment~\cite{schedl2018current}. 
This paper studies such a setting where decisions propagate through latent dynamics, leading to correlations that fundamentally change both how to learn and how to act optimally.

We study a nonstationary bandit problem in which reward depends \emph{bilinearly} on the \emph{current action} and an \emph{unobserved latent state} that evolves according to a stable linear dynamical system. 
Formally,
\begin{align}\label{eqn:sys}
    r_t = \vu_t^\top \vC \vx_t + z_t,
    \quad
    \vx_{t+1} = \vA \vx_t + \vB \vu_t + \vw_t,
\end{align}
where $\vA,\vB,\vC$ are unknown matrices, $\vu_t$ is a bounded action chosen by the learner, $\vw_t,z_t$ are random noise processes, and $\vx_t$ is the latent state, at time $t{\geq}0$. 
Unlike classical multi-armed bandits with i.i.d.\ rewards, here actions not only determine immediate payoffs but also propagate through the state dynamics, creating temporal correlations in rewards, and making parameter estimation and optimal action sequence search significantly more challenging. 

We observe that the problem can be resolved in the following ways. 
The issue of temporal correlation to estimate unknown parameters is addressed by improving upon recent results from system identification \cite{sattar2025learning}.
While the state-space parameters may not be identifiable from the observed rewards~(without additional assumptions on $\vA,\vB,\vC$), it is possible to obtain high-probability bounds on the estimation error for the \emph{Markov parameters}, which characterize the action-reward relationship. 
In order to obtain an optimal long-term reward, we show that it is sufficient to use only this action-reward representation.
In particular, selecting an optimal action sequence is equivalent to solving an indefinite quadratic problem over a hypercube, where the quadratic function is defined in terms of the Markov parameters. 

Combining these two insights, we propose an \emph{explore--then--commit} (ETC) algorithm. 
In the exploration phase, the learner applies random Rademacher actions to estimate the system’s Markov parameter.
In the commit phase, it commits to an action sequence which is a solution to an indefinite quadratic optimization problem. 
We show that, with high probability, this algorithm achieves sublinear regret, scaling as $\tilde{\Ocal}(T^{2/3})$.

Our setting bears relation to several other nonstationary bandits,
including restless bandits~\cite{whittle1988restless},
rebounding bandits~\cite{leqi2021rebounding}, and bandits with underlying state~\cite{khosravi2023bandits}, to name a few (see \textsection\ref{sec:related-work}  for further discussion).
Key novelties of our setting include that
states can be of arbitrary dimension, the evolution
can depend on interactions between dimensions (i.e., the matrix $\vA$ need not be diagonal),
and actions directly affect states in a potentially correlated manner (through the matrix $\vB$).
Furthermore, unlike many multi-armed bandit settings, we consider a continuous action space.
Our problem also bears relation to model-based reinforcement learning for linear dynamics~\cite{dean2018regret,simchowitz2020making,mania2019certainty,lale2020regret,lale2020logarithmic}, but the reward differs and provides only partial information about the underlying state.

\noindent{\bf Contributions:} In this paper, we make following contributions: 
\begin{itemize}[leftmargin=*,noitemsep,topsep=0pt]
\item {\bf Framework:} We propose a novel nonstationary bandit problem in which the action affects both the current reward (through a bilinear interaction) as well as the future rewards (through a latent state).
\item {\bf Algorithm:} We propose an explore-then-commit algorithm for our bandit problem.  
After exploration with random Rademacher actions, we estimate the Markov parameters which characterize the action-reward relationship. We then commit to an optimized action sequence obtained by solving a semidefinite relaxation of indefinite quadratic optimization over a hypercube --- a known NP-hard problem.

\item {\bf Estimation:} We provide near-optimal sample complexity and estimation error bounds for learning Markov parameters. Specifically, in terms of the exploration length $H$, our error rate scales as $\tilde{\Ocal}(1/\sqrt{H})$, and we require $H \gtrsim \tilde{\Ocal}(d_M)$, where $d_M$ is the dimension of unknown Markov parameters. 
\item {\bf Regret:} We prove an upper bound on regret scaling as $\tilde{\Ocal}(T^{2/3})$, and additionally provide a sub-optimality guarantee for the semidefinite relaxation with the Goemans-Williamson rounding approach.
\end{itemize}

\paragraph{Organization:}
\textsection\ref{sec:problem-formulation} formalizes the model and the regret benchmark, defines the action set, and discuss the computational difficulty of the problem.
\textsection\ref{sec:main-results} introduces the ETC algorithm and its regret guarantee. 
\textsection\ref{sec:sys-id} presents our parameter estimation results with sample complexity and error bounds. 
\textsection\ref{sec:regret-analysis} discuss our regret analysis. 
\textsection\ref{sec:exp} presents numerical experiments. 
\textsection\ref{sec:related-work} shows related work to this paper, and \textsection\ref{sec:conclusion} concludes the paper.

\noindent{ \bf Notations:} We use boldface lowercase/uppercase letters to denote vectors/matrices. 
The $\ell_2$-norm and $\ell_\infty$-norm of a vector $\vx$ are denoted by $\twonorm{\vx}$ and $\infnorm{\vx}$, respectively.
The spectral radius, the spectral norm, and the Frobenius norm of a matrix $\vX$ are denoted by $\rho(\vX), \norm{\vX}$, and $\fronorm{\vX}$, respectively.
The largest and smallest eigenvalue of a square matrix $\vX$ are denoted by $\lambda_{\max}(\vX)$ and $\lambda_{\min}(\vX)$. 
The operation $\otimes$ denotes the Kronecker product.
We use $\gtrsim$ and $\lesssim$ for inequalities that hold up to an absolute constant factor. 
The notation $\tilde{\Ocal}$ hides constants and logarithmic terms. 
Lastly, we use $\vzero_d$ and $\vzero_{m\times n}$ to denote the zero vector in $\R^d$ and the zero matrix in $\R^{m\times n}$, respectively.

\section{Problem Formulation} \label{sec:problem-formulation}

\input{graphical_model}

We consider a nonstationary stochastic bandit problem with controlled\footnote{i.e., the state is nontrivially influenced by the actions} latent dynamics and a bilinear reward model.
At each round $t{=}0,1, \dots, T$, the learner selects an action $\vu_t$ from a bounded action set $\Ucal {\subseteq} \R^p$, and receives a reward $r_t {\in} \R$. 
More specifically, the reward is bilinear in the latent state $\vx_t {\in} \R^n$ and the action $\vu_t {\in} \Ucal$, whereas the latent state  evolves according to a linear dynamical system, 
as defined in~\eqref{eqn:sys},
where $\vA, \vB, \vC$ are unknown matrices of appropriate dimensions,  and $\vw_t$, $z_t$ are random zero-mean noise processes. 
Without loss of generality, we assume $\vx_0{=}\vzero_n$. 
Throughout, we assume $\vA$ is Schur-stable, that is, $\rho(\vA){<}1$, and the noise processes are i.i.d. zero-mean sub-Gaussian, with $\vw_t$ having variance proxy $\vSigma_w$ and $z_t$ having variance proxy $\sigma_z^2$.

\subsection{Objective, Action Set, and Regret}
Our objective is to maximize the expected cumulative reward $\E\left[ \sum_{t=0}^T r_t\right]$ over the horizon $T$, by choosing actions from a bounded set.  
For simplicity, we consider the action set to be a centered hypercube in $\R^p$. In other words, we choose
\begin{align}
    \Ucal =\{\vu \in \R^p \colon \infnorm{\vu}\le 1 \} = [-1,1]^p. \label{eqn:Ucal_def}
\end{align}
Let $\{r_t^\star\}_{t=0}^T$ denote the rewards collected under the \emph{optimal open-loop action sequence}, and $\{r_t^\pi\}_{t=0}^T$ denote the rewards collected by policy $\pi$. 
Then, the regret up to round $T$ of a policy $\pi$ is defined as:
\begin{align} \label{eqn:regret_def}
    R_T(\pi) := \E \left[ \sum_{t=0}^T r_t^\star - \sum_{t=0}^T r_t^\pi\right]. 
\end{align}

\subsection{Optimal Open-Loop Actions}\label{sec:open-loop}

In this section, we suppose that $\vA, \vB, \vC$ are known, and show that the optimal action sequence is the solution to an indefinite quadratic optimization problem with $\ell_{\infty}$-norm constraint.
By unrolling the latent state through~\eqref{eqn:sys}, we have $r_0 = z_0$, and for $t \ge 1$,
\begin{align}
    r_t = \vu_t^\top \vC \sum_{i=0}^{t-1} \vA^{t-i-1} (\vB\vu_i + \vw_i) + z_t.
\end{align}
Since $\vw_t$ and $z_t$ have zero-mean, the expected cumulative reward is given by
\begin{align} 
    \E \left[ \sum_{t=0}^T r_t\right] = \vu_{0:T}^\top \vM_T \vu_{0:T},\label{eqn:exp-cumul-reward}
\end{align}
where $\vu_{0:T}:= \left[ \vu_T^\top ~~ \vu_{T-1}^\top ~~ \cdots ~~ \vu_0^\top\right]^\top \in \R^{p(T+1)}$, 
and $\vM_T \in \R^{p(T{+1}) \times p(T{+1})}$ is a block Toeplitz matrix with $(i,j)$-th $p\times p$ block given by
\begin{align} \label{eqn:block-toeplitz-M}
    (\vM_T)_{ij} := 
    \begin{cases}
        \vC\vA^{j-i-1}\vB & (i<j) \\
        \vzero_{p\times p} & (i \ge j).
    \end{cases}
\end{align}
With these definitions, the optimal action sequence and the corresponding cumulative reward are given as follows. 
This cumulative reward is our regret baseline.

\begin{proposition} \label{prop:regret-benchmark}
    Let $\vS_T {:=} \vM_T {+} \vM_T^\top$. Then, the optimal open-loop action sequence $\vu_{0:T}^\star$ is the solution to the problem:
    \begin{equation} \label{eqn:baseline-regret-form}
        \begin{split}
            \max_{\vu_{0:T}} & \quad \frac{1}{2}\vu_{0:T}^\top \vS_T \vu_{0:T} \\
            \text{subject to}& \quad \vu_t \in [-1,1]^p, \quad \forall\, t=0,1,\dots,T
        \end{split}
    \end{equation}
\end{proposition}
Note that the maximum value of \eqref{eqn:baseline-regret-form} (i.e. the optimal expected cumulative reward) is at least
$\frac{1}{2} \lambda_{\max}\left(\vS_T\right)$ 
because $\twonorm{\vu_{0:T}} \le 1$ implies $\infnorm{\vu_{0:T}}\leq 1$.

This optimal action sequence represents the optimal \emph{open-loop} strategy for accruing reward,
in contrast to a \emph{feedback} or \emph{closed-loop} policy \cite{bar2003dual}.
It captures the problem of selecting optimal \emph{sequences} of actions to maximize the long-term reward, a consideration which is not present in the classical bandit settings.

\subsection{Combinatorial Actions} \label{subsec:action-set-NP-hard} 

The problem~\eqref{eqn:baseline-regret-form} is the maximization of an indefinite quadratic function over the hypercube.
We show that it suffices to consider a discrete set of actions corresponding to the vertices of the hypercube,
\begin{align}
    \Ucal = \{-1,+1\}^p
\end{align}
so that each coordinate of every action is restricted to $\pm 1$. 

\begin{proposition} \label{prop:max-vertex}
    If $\vA \in \R^{n\times n}$ is symmetric with nonnegative diagonal entries, a maximizer of $\vz^\top \vA \vz$ over $[-1,1]^n$ exists at a vertex $\vs \in \{-1,+1\}^n$, that is,
    \begin{align}
        \max_{\|\vz\|_\infty\le1} \vz^\top\vA\vz \;= \max_{\vs \in \{-1,+1\}^n} \vs^\top\vA\vs.
    \end{align}
\end{proposition}

The proof of Proposition~\ref{prop:max-vertex} is deferred to the appendix.  
We therefore take the action set to be $\Ucal=\{-1,+1\}^p$ instead of the hypercube $[-1,1]^p$ in the remainder of this paper.
Under this choice of action set, the optimization~\eqref{eqn:baseline-regret-form} reduces to an instance of \emph{quadratic unconstrained binary optimization} (QUBO), which is known to be NP-hard, as it generalizes classical problems such as MaxCut. Consequently, solving for the optimal actions exactly is computationally intractable.
This motivates the use of semidefinite relaxations and randomized rounding schemes in \textsection\ref{sec:exp}.

\section{Main Result: Explore-then-Commit} \label{sec:main-results}

We propose an explore–then–commit (ETC) algorithm for minimizing regret tailored to bandits with latent linear dynamics and bilinear rewards. The ETC algorithm runs in two phases, \emph{an exploration phase} and \emph{a commit phase}.  In \textsection\ref{subsec:exploration-phase}, we describe the \emph{exploration phase}, which consists in estimating the so-called \emph{Markov parameters}\footnote{The Markov parameters corresponding to the system \eqref{eqn:sys} are typically defined, for example in \citet{sattar2025learning}, as the sequence of matrices $\lbrace \vC \vA^k \vB\rbrace_{k \ge 0}$.} of the system. In \textsection\ref{subsec:commit-phase}, we describe the \emph{commit phase}, where the learner uses the estimated  Markov parameters to formulate an open-loop optimization problem and design a sequence of actions for the remaining horizon. Finally, in \textsection\ref{subsec:regret-guarantee}, we state our main theoretical result: with appropriate choices of exploration and truncation lengths, the ETC algorithm achieves a regret of order $\tilde{\Ocal}(T^{2/3})$. We summarize the overall procedure in Algorithm~\ref{alg:EtC}.

\begin{algorithm}[ht]
\caption{Explore-then-Commit}
\label{alg:EtC}
\begin{algorithmic}[1]
\Require Horizon $T$, exploration length $H$, truncation length $L$

\Statex{{\color{blue}\underline{\textbf{I. Exploration Phase}} }} 

\State For $t {=} 0, \dots, H $, play $\vu_t {\sim} \unif(\{-1,+1\}^p)$ and observe reward $r_t$.

\State Find
an estimate $\vGhat$ of the first $L$ Markov parameters via least squares using $ \{(\vu_t, r_t)\}_{t=0}^H$ as in \eqref{eqn:least_squares}.

\State Construct $\vShat_{T-H-1} = \frac{1}{2}(\vMhat_{T-H-1} + \vMhat_{T-H-1}^\top)$ where $\vM_{T-H-1}$ is defined via $\vGhat$ as in \eqref{eqn:block-toeplitz-M}.

\Statex{\underline{\blue{\textbf{II. Commit Phase}}}}
\State Find a sequence $(\vu^\pi_{H+1}, \dots, \vu^\pi_{T})$ that solves,
\vspace{-5pt}
\begin{align*}
    \max_{\vu_{H+1:T}}& \tfrac{1}{2}\vu_{H+1:T}^\top \vShat_{T-H-1} \vu_{H+1:T} \\
    \text{s.t. }& \, \vu_t \in \{-1,+1\}^p, \;\; \forall \, t=H+1,\dots,T.
\end{align*}
by {\bf (a)} SDP relaxation with Goemans-Williamson rounding, or {\bf (b)} sign-iteration method.

\State For $t {=} H{+}1, \dots, T $, play  $\vu_{t}^\pi$ and observe reward $r_t$.
\end{algorithmic}
\end{algorithm}

\subsection{Exploration Phase} \label{subsec:exploration-phase}
In the exploration phase, the learner selects actions independently from the Rademacher distribution $\unif\{-1,+1\}^p$. This distribution satisfies the action constraint (see \textsection\ref{subsec:action-set-NP-hard}) and provides sufficient excitation of the system, ensuring that the Markov parameters can be consistently estimated~\cite{sattar2025learning}. Further discussion of persistence of excitation is deferred to the appendix.

The trajectory \(\{(\vu_t,r_t)\}_{t=0}^H\) collected in this phase is used to estimate the system’s Markov parameters. Concretely, we form nonlinear regressors/features using the previous $L$ actions $(\vu_{t-1},\dots,\vu_{t-L})$ together with the current action $\vu_t$, and use linear regression to predict the current reward $r_t$. The least-squares solution gives an estimate of the first $L$ Markov parameters $\{\vC\vA^k\vB\}_{k=0}^{L-1}$.

Since the matrix $\vA$ is Schur-stable, the influence of terms beyond lag $L$ decays geometrically. Hence, it suffices to estimate only the first $L$ parameters. The details of the regression procedure and error analysis are deferred to \textsection\ref{sec:sys-id}.

\subsection{Commit Phase} \label{subsec:commit-phase}

After exploration, the learner uses the estimated Markov parameters to construct an open-loop optimization problem aimed at maximizing the expected cumulative reward from $t=H{+}1$ to $T$. 
Because the Rademacher actions and the noise processes are zero-mean, the theoretical objective will be
\begin{align} \label{eqn:commit-theoretical-objective}
    \E\left[\sum_{t=H+1}^T r_t \right] = \frac{1}{2}\vu_{H+1:T}^\top \vS_{T\!-\!H\!-\!1} \vu_{H+1:T} 
\end{align}
with $\vS_{T\!-\!H\!-\!1} = \vM_{T\!-\!H\!-\!1} + \vM_{T\!-\!H\!-\!1}^\top$ where $\vM_{T\!-\!H\!-\!1}$ is a block Toeplitz matrix with similar structure to  Equation \eqref{eqn:block-toeplitz-M} (see the appendix for derivation).

Since the true parameters are unknown, we replace them with estimates from the exploration phase. 
Specifically, we form $\vShat_{T-H-1}$ by substituting the estimated blocks for $\vC \vA^k \vB$ whenever $k < L$ and setting the blocks to $\vzero_{p\times p}$ for $k \ge L$, consistent with the Schur-stability of $\vA$. This yields the commit-phase optimization problem:
\vspace{-3pt}
\begin{equation} \label{eqn:commit-obj}
    \begin{split}
        \max_{\vu_{H+1:T}}& \quad \frac{1}{2}\vu_{H+1:T}^\top \vShat_{T-H-1} \vu_{H+1:T} \\
        \text{subject to}&  \quad  \vu_{H+1:T} \in \{-1,+1\}^{p(T-H)}
    \end{split}
\end{equation}
\vspace{-3pt}

The structure of this problem is identical to \eqref{eqn:baseline-regret-form}, but with estimated rather than true parameters. As noted in \textsection\ref{subsec:action-set-NP-hard}, the problem \eqref{eqn:commit-obj} is NP-hard. We therefore rely on practical approaches to obtain tractable approximations, which will be discussed in \textsection\ref{subsec:SDP-GW} and \textsection\ref{subsec:sign-iter}. 

\subsection{Regret Guarantee} \label{subsec:regret-guarantee}

Our main result establishes that the ETC algorithm achieves sublinear regret with high probability. 

\begin{theorem} \label{thm:regret-bound}
    With exploration length $H=\tilde{\Ocal}(T^{2/3})$ and truncation length $L=\Theta(\log T)$, the explore-then-commit algorithm $\pi$~(Alg.~\ref{alg:EtC}) achieves regret 
    \begin{align*}
        R_T(\pi) = \tilde{\Ocal}(T^{2/3})
    \end{align*}
    with high probability.
\end{theorem}

The proof of Theorem~\ref{thm:regret-bound} is discussed in \textsection\ref{sec:regret-analysis}. Note that the regret inevitably has a high-probability bound because the Markov parameter estimation error, which affects the regret, has a high-probability bound.

This rate arises from a simple trade-off. The algorithm incurs linear loss during exploration ($\sim H$) and suffers estimation error of order $\tilde{\Ocal}(\sqrt{1/H})$ across the remaining horizon ($\sim T$).
Optimizing the bound over $H$ yields the $\tilde{\Ocal}(T^{2/3})$ rate with $H=\tilde{\Ocal}(T^{2/3})$. 
Lastly, taking $L=\Theta(\log T)$ makes the truncation error negligible, as already noted in \textsection\ref{subsec:exploration-phase} where Schur stability implies geometric decay of higher-order terms.

\section{Main Result: Parameter Estimation} \label{sec:sys-id}
In this section, we use the action-reward samples $\{(\vu_t, r_t)\}_{t=0}^{H}$ from the exploration phase to estimate a map between the actions and rewards. Similar to \citet{sattar2025learning}, we can estimate Markov parameters by regressing the rewards $r_t$ to an expression defined by the history of inputs $\{\vu_\tau\}_{\tau \leq t}$. For $t \geq L$, the reward $r_t$ depends on the past $L$ actions as follows:
\begin{align}
    r_t &= \vu_t^\top \vC\vA^L\vx_{t-L} \nonumber \\
        &\quad + \vu_t^\top\vC \sum_{i=0}^{L-1} \vA^i(\vB\vu_{t-i-1} \!+\! \vw_{t-i-1}) + z_t. \label{eqn:reward_expansion}
\end{align}
To ease notation, for any sequence of vectors $\{\vq_t\}_{t=0}^T$, letc$\bar{\vq}_t:=\left[\vq_t^\top~~\vq_{t-1}^\top~~\cdots~~\vq_{t-L+1}^\top\right]^\top$ denote a concatenation of past $L$ vectors starting from $t {\geq} L$. If we let
\begin{align}
    \vG := 
    \begin{bmatrix}
        \vC\vB & \vC\vA\vB & \cdots & \vC\vA^{L-1}\vB 
    \end{bmatrix} \in \R^{p \times pL}
\end{align}
then \eqref{eqn:reward_expansion} can be compactly written as
\begin{align}
    r_t &= \vu_t^\top \vG \vubar_{t-1} + \vu_t^\top \vC\vA^L \vx_{t-L} + \vu_t^\top \vF \vwbar_{t-1} + z_t \nonumber \\
    &= \vvec(\vG)^\top \vutil_t + \zeta_t
    \label{eqn:reward_covariate_map}
\end{align}
where we define the covariates $\vutil_t {:=} \vubar_{t-1} \otimes \vu_t \in \R^{p^2L}$,  the effective noise $\zeta_t {:=} \vu_t^\top \vC\vA^L \vx_{t-L} {+} \vu_t^\top \vF \vwbar_{t-1} {+} z_t$, and $\vF {:=} [\vC~~ \vC\vA ~~ \cdots ~~\vC\vA^{L-1}]$. Hence, we can formulate the following least-squares problem to estimate the unknown parameter $\vG$ from $\{(\vu_t, r_t)\}_{t=0}^{H}$:
\begin{align}
    \hat{\vG}= \argmin_{\vG \in \R^{p \times pL}}  \sum_{t=L+1}^H 
    \left(r_t {-} \vvec(\vG)^\top \vutil_t \right)^2. \label{eqn:least_squares}
\end{align}
Let $\vUtil {:=} \sum_{t=L+1}^H \vutil_t\vutil_t^\top$ for brevity. When $\vUtil$ is full rank, the solution to the least-squares problem above is $\vvec(\vGhat) = \vUtil^{-1} \sum_{t=L+1}^H \vutil_t r_t$, and the corresponding estimation error is given by
\begin{align}
    \vvec(\vGhat) - \vvec(\vG) = \vUtil^{-1} \sum_{t=L+1}^H \vutil_t  \zeta_t. \label{eqn:estimation_error_form}
\end{align}
In this paper, we provide improved sample complexity and error bounds (compared with \citet{sattar2025learning}), under the following assumption.
\begin{assumption}\label{assump:sysID}
    {\bf (a)} $\vA$ is Schur-stable; 
    {\bf (b)} $\vw_t$ and $z_t$ are i.i.d. centered sub-Gaussian with variance proxy $\vSigma_w$ and $\sigma_z$, respectively; {\bf (c)}  $\vu_t \distas  \unif\left(\{-1,+1\}^p\right)$.
\end{assumption}

Under Assumption~\ref{assump:sysID}, we derive a bound on sample complexity and an estimation error of \eqref{eqn:estimation_error_form} in Theorem~\ref{thm:sysID}. We note the following fact regarding the stability condition.
According to Gelfand's formula, for all $\rho > \rho(\vA)$, the quantity $\phi(\vA,\rho) {:=} \sup_{k \in \mathbb{Z}_+}(\|\vA^k\|/ \rho^k)$ is finite. Hence, if $\rho(\vA) {<} 1$, for all $\rho {\in} (\rho(\vA), 1)$, we have $\|\vA^k\| {\leq} \phi(\vA,\rho) \rho^k$ for all $k {\in} \mathbb{Z}_+$. A proof is shown in the appendix.

\begin{theorem}\label{thm:sysID}
    Under Assumption~\ref{assump:sysID}, let $\{(\vu_t,r_t)\}_{t=0}^H$ be a single trajectory of action-reward pairs collected from the system \eqref{eqn:sys}.
    For given $\delta \in (0,1)$, suppose 
    \begin{align*}
        H-L \gtrsim (L+1) \left( p^2L \log(p^2L) + \log \left(\frac{L+1}{\delta} \right)\right).
    \end{align*}
    With probability at least $1-\delta$, we have $  \|\vGhat - \vG\|_F \lesssim$
    \begin{align*}
         &\frac{\phi(\vA,\rho) \rho^L \|\vC\|}{1-\rho} \sqrt{\frac{p^3L\lambda_{\max}(\vB\vB^\top {+} \vSigma_w)\log(1/\delta)}{H{-}L}}\\
         &+ (\sigma_z {+} \sqrt{p}~\Lambda_w^{(L)}) \sqrt{\frac{p^2L\log (p^2L(H{-}L)) {+} \log(L/\delta)}{H{-}L}}
    \end{align*}
    for $\Lambda_w^{(L)} {:=} \sum_{k=0}^{L-1}\sqrt{\lambda_{\max}(\vC\vA^k \vSigma_w (\vA^k)^\top \vC^\top)}$.
\end{theorem}

The proof of Theorem~\ref{thm:sysID} is presented in the appendix. The recent work~\cite{sattar2025learning} shows that such a randomized design yields an estimation error bound of order $\tilde{\Ocal}(1/\sqrt{\delta H})$ with probability at least $1-\delta$. In this work, we sharpen this guarantee by exploiting the sub-Gaussian structure of the noise processes, obtaining an improved bound of order $\tilde{\Ocal}(\sqrt{\log(1/\delta)/H})$ with probability at least $1-\delta$. 

The first term in our error bound corresponds to the approximation error because we use only $L$ Markov parameters to reconstruct the reward function. It decays exponentially with $L$ when the latent dynamics is strictly stable because the norm of Markov parameters decreases exponentially. The second term corresponds to the error due to noisy reward model and noisy latent dynamics update. This term depends linearly on the noise variance proxies.   

The sample complexity in Theorem~\ref{thm:sysID} scales as $\tilde{\Ocal}(p^2L^2)$, 
which is optimal in the dimension of Markov parameters $p^2$. In contrast, the sample complexity in \citet{sattar2025learning} scales as $\tilde{\Ocal}(p^4L^3)$. The reason for this improvement is the choice of inputs $\vu_t \distas  \unif\left(\{-1,+1\}^p\right)$ which leads to better persistence of excitation. 
Secondly, the error bound in Theorem~\ref{thm:sysID} depends on the failure probability $\delta$ through $\log(1/\delta)$. This is significantly better than the $1/\delta$ dependence in \citet{sattar2025learning} which considers a setting where noise processes can be heavy-tailed. 
The sub-Gaussian noise assumption in this paper allows using better concentration arguments, such as self-normalized bounds for martingales and Freedman’s inequality, to obtain optimal dependence on the failure probability.

\section{Regret Analysis} \label{sec:regret-analysis}

In this section we prove Theorem~\ref{thm:regret-bound}. Additional details and proofs of intermediate results are deferred to the appendix. 
Recall the definition of regret from \eqref{eqn:regret_def}.
From Proposition~\ref{prop:regret-benchmark} and Equation~\eqref{eqn:commit-theoretical-objective}, we obtain
\begin{align*}
    R_T(\pi) {=} \frac{1}{2} (\vu_{0:T}^\star)^{\!\top} \vS_T \vu_{0:T}^\star 
    \!-\! \frac{1}{2}(\vu_{H\!+\!1:T}^\pi)^{\!\top} \vS_{T\!-\!H\!-\!1} \vu_{H\!+\!1:T}^\pi
\end{align*}
where $\vu_{H+1:T}^\pi$ is the optimized action sequence which maximizes (\ref{eqn:commit-obj}).
For the purpose of analysis, let $\vutil_{H+1:T}$ be an action sequence that maximizes Equation~\eqref{eqn:commit-theoretical-objective} with true parameters under the constraint $\vutil_t \in \{-1,+1\}^p$ for $t=H+1,\dots,T$.
Then, we decompose the regret as follows:
\begin{align}
    R_T(\pi) = \frac{1}{2} \left( R_{1,T} + R_{2,T} + R_{3,T}\right)
\end{align}
where
\begin{align*}
    R_{1,T} &:= (\vu_{0:T}^\star)^\top \vS_T \vu_{0:T}^\star - (\vutil_{H+1:T})^\top \vS_{T-H-1} \vutil_{H+1:T} \\
    R_{2,T} &:= (\vutil_{H+1:T})^\top \vS_{T-H-1} \vutil_{H+1:T}  \\
    &\qquad \qquad \quad - (\vu_{H+1:T}^\pi)^\top \vShat_{T-H-1} \vu_{H+1:T}^\pi \\
	R_{3,T} &:= (\vu_{H+1:T}^\pi)^\top (\vShat_{T-H-1}-\vS_{T-H-1}) \vu_{H+1:T}^\pi.
\end{align*}
Intuitively, $R_{1,T}$ captures the sub-optimality
between the full-time horizon and the commit-phase horizon, $R_{2,T}$ captures the sub-optimality between the true and estimated dynamics, and finally $R_{3,T}$ captures the error from parameter estimation. We analyze these three terms in detail in \textsection\ref{subsec:analysis-a} and \textsection\ref{subsec:analysis-b-c}.

\subsection{Upper Bounding $R_{1,T}$} \label{subsec:analysis-a}

The term $R_{1,T}$ quantifies the loss from truncating the horizon at $H$, capturing both short-term dependencies within the exploration phase and long-term dependencies between exploration and commit phases through the dynamics.

\begin{proposition} \label{prop:regret-a}
    Let $\rho\in(\rho(\vA),1)$ be given. Then
    \begin{align*}
        R_{1,T} \leq 2p\kappa^2 \left( \alpha H + \beta \right) ,
    \end{align*}
    for $\alpha = 1+\frac{\phi(\vA,\rho)\rho}{1-\rho}$, $\beta = \frac{\phi(\vA,\rho)\rho}{(1-\rho)^2} +1$, and $\kappa =\max\{ \|\vB\|, \|\vC\|\}$.
\end{proposition}

The proof of Proposition~\ref{prop:regret-a} is deferred to the appendix. 
Note that the upper bound in \ref{prop:regret-a} depends linearly on $H$, implying that the available reward grows linearly in time horizon.
There is also a second term, a constant term independent of $H$.
The coefficients $\alpha$ and $\beta$ depending on the stability of the state dynamics $\vA$, and can be understood as effective memory capacity~\cite{kumar2024online}.
The first term arises from to short-term dependencies within the exploration-phase. The second term is related to long-term dependencies which, due to stability, do not depend on the horizon.

\subsection{Upper Bounding $R_{2,T}$ and $R_{3,T}$} \label{subsec:analysis-b-c}

The terms $R_{2,T}$ and $R_{3,T}$ both depend on the estimation error $\epsilon$ and the truncation error $\rho^L$, and their bounds are given in the following proposition. 

\begin{proposition} \label{prop:regret-b-c}
    Let $\rho \in (\rho(\vA),1)$ be given and $\epsilon>0$ be the high-probability parameter estimation error, 
    i.e., $\|\vG-\vGhat\|_F \le \epsilon$. Then, with high probability, 
    \begin{align*}
        \max\{R_{2,T}, R_{3,T} \} \leq 2p(T-H) \left( \epsilon + \kappa^2 \gamma_L\right)
    \end{align*}
    where $\kappa = \max\{\|\vB\|, \|\vC\|\}$ and $\gamma_L = \frac{\phi(\vA,\rho)\rho^L}{1-\rho}$.
\end{proposition}

The proof of Proposition~\ref{prop:regret-b-c} is deferred to the appendix. 
We can observe that the bound depends on the length of commit-phase $T-H$, the estimation error $\epsilon$, and the truncation error $\rho^L$. 

\subsection{Final Regret Analysis}

Finally, combining Propositions~\ref{prop:regret-a} and \ref{prop:regret-b-c} yields the following high-probability bound:
\begin{align}
R_T(\pi) \le 2p\kappa^2(\alpha H + \beta) + 4pT(\epsilon + \kappa^2\gamma_L).
\end{align}
From Theorem~\ref{thm:sysID}, we can take  $\epsilon \lesssim \sqrt{\tfrac{\log(1/\delta)}{H-L}}$ with probability at least $1-\delta$. 
Optimizing over $H$ gives $H=\tilde{\Ocal}(T^{2/3})$, and with $L=\Theta(\log T)$ we obtain the high-probability regret bound $R_T(\pi)=\tilde{\Ocal}(T^{2/3})$.
Note that we define regret to be  the expected cumulative rewards, where expectation is taken both over actions (in the exploration phase) and noise processes. 
However, we still have a high-probability bound on the regret because the estimation error bound holds with high probability.

\section{Numerical Experiments} \label{sec:exp}

\begin{figure*}[ht]
    \centering
    \begin{subfigure}[t]{0.24\textwidth}
        \centering
        \includegraphics[width=\linewidth]{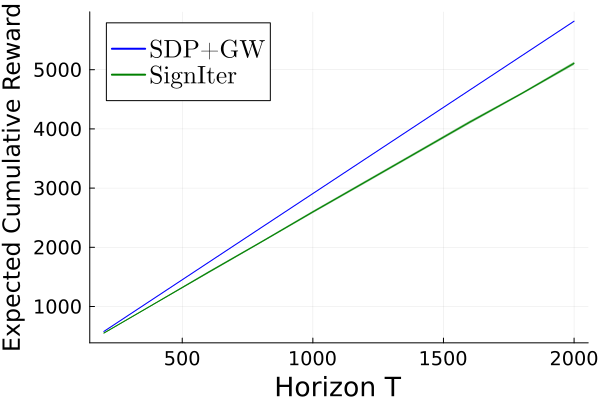}
        \caption{Oracle}
        \label{fig:regret-benchmark}
    \end{subfigure}
    \hfill
    \begin{subfigure}[t]{0.24\textwidth}
        \centering
        \includegraphics[width=\linewidth]{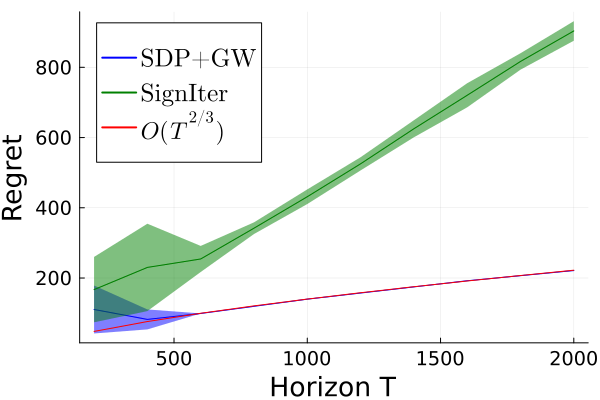}
        \caption{Regret}
        \label{fig:regret}
    \end{subfigure}
    \hfill
    \begin{subfigure}[t]{0.24\textwidth}
        \centering
        \includegraphics[width=\linewidth]{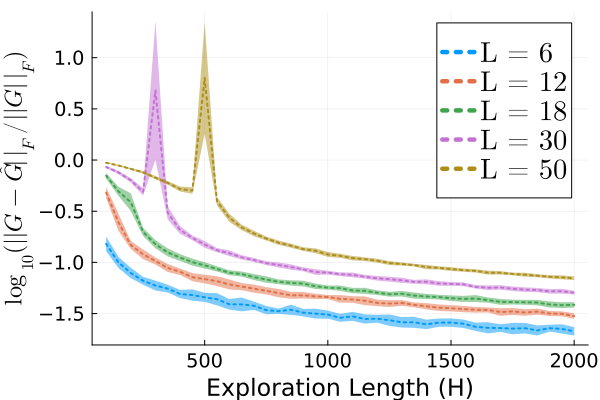}
        \caption{$\rho(\vA)=0.1$}
        \label{fig:param-est-rho-p1}
    \end{subfigure}
    \hfill
    \begin{subfigure}[t]{0.24\textwidth}
        \centering
        \includegraphics[width=\linewidth]{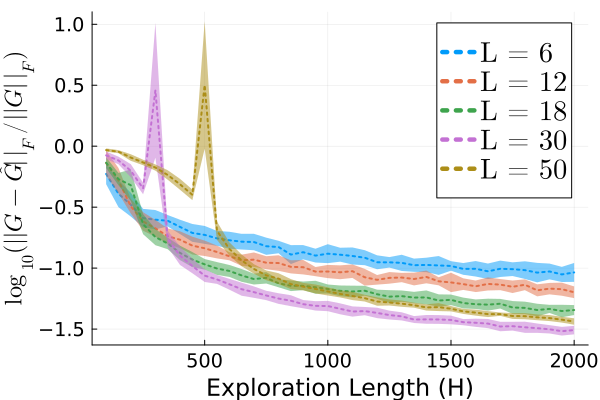}
        \caption{$\rho(\vA)=0.9$}
        \label{fig:param-est-rho-p9}
    \end{subfigure}
    \caption{Each curve shows a mean over 20 experiments, with shaded regions indicating $\pm1$ standard deviation. 
    (a) Expected cumulative reward under the oracle benchmark, approximated by semidefinite relaxation with Goemans-Williamson rounding (\texttt{SDP+GW}) and by the sign-iteration method (\texttt{SignIter}). (b) The regret of the explore-then-commit algorithm measured against the \texttt{SDP+GW} oracle benchmark, compared with the theoretical $\tilde{\Ocal}(T^{2/3})$ rate. 
    (c)--(d) Relative error of Markov parameter estimation for different truncation lengths $L$, under systems with spectral radii $\rho(\vA)=0.1$ and $\rho(\vA)=0.9$.}
    \label{fig:param-est}
\end{figure*}

In this section, we present numerical experiments with synthetic data. 
To obtain tractable solutions to \eqref{eqn:baseline-regret-form} and \eqref{eqn:commit-obj} and analyze sub-optimality, we investigate two methods: 
(i) a semidefinite relaxation combined with Goemans-Williamson random hyperplane rounding, and 
(ii) a heuristic sign-iteration method.
The code for experiments can be found in \url{https://github.com/sdean-group/EtC-latent-bandits}.

\subsection{SDP Relaxation and Rounding} \label{subsec:SDP-GW}

Our problem is to maximize a quadratic form $\vx^\top \vW \vx$ for some symmetric $\vW \in \R^{n\times n}$ over $\vx\in \{-1,1\}^n$, which is equivalent to
\begin{equation} \label{eqn:simple-form-matrix}
    \begin{split}
        \text{maximize } \; & \text{tr}(\vW\vX) \\
        \text{subject to } \; & \vX \succeq 0, \; \text{rank} (\vX) =1 \\
        & \vX_{ii} =1, \quad i=1,\dots,n.
    \end{split}
\end{equation}
Dropping the rank constraint yields the semidefinite relaxation, which we solve using the Mosek solver \cite{mosek}.

To obtain a feasible binary matrix from a solution $\vX$ to the relaxed problem, we apply Goemans-Williamson (GW) random hyperplane rounding algorithm \cite{goemans1995improved}: factor $\vX = \vV^\top \vV$, sample $\vr \sim \Ncal(0,\vI_n)$, and set a vector
\begin{align*}
    \vx = 
    \begin{bmatrix}
        \text{sign}(\vr^\top \vv_1) & \cdots & \text{sign}(\vr^\top \vv_n)
    \end{bmatrix}^\top
\end{align*}
where $\vv_i$ is the $i$th column vector of $\vV$.
Then the matrix $\vx\vx^\top$ is rank-one and feasible for \eqref{eqn:simple-form-matrix}. In practice, we repeat the rounding multiple times and keep the best value of $\vx^\top \vW \vx$.
We will hereafter refer to this method as \texttt{SDP+GW}.

For MaxCut problems, \texttt{SDP+GW} achieves an $\alpha$-approximation algorithm with $\alpha\approx0.87856$. While our objective is not exactly MaxCut,
we derive a similar lower bound which depend on $\vW$ (see the appendix). 

\subsection{Sign-Iteration Method} \label{subsec:sign-iter}

As an alternative to relaxation and rounding, we also consider a heuristic method called \emph{sign-iteration} (henceforth \texttt{SignIter}) to maximize $\vx^\top \vW \vx$ over $\vx \in \{-1,+1\}^n$. Starting from a random $\vx^{(0)}$, the update is
\[
x^{(k+1)}_i = 
\begin{cases}
\mathrm{sign}\big((\vW \vx^{(k)})_i\big), & (\vW \vx^{(k)})_i \neq 0, \\
x^{(k)}_i, & \text{otherwise},
\end{cases}
\]
and the procedure repeats until convergence or a maximum number of iterations. To mitigate dependence on initialization, the method is run multiple times and the best objective value is returned.

\subsection{Experimental Evaluation of Regret} \label{subsec:exp-regret}
We consider a simple latent dynamics and reward function specified by
\iffalse
\begin{align*}
    \vA{=} \diag([0.3~0.15~0.12]), ~ \vB {=} \begin{bmatrix}
           \vI_2 \\
            0.5 ~ 0.4
        \end{bmatrix}, ~ \vC {=} \begin{bmatrix}
           \vI_2 \\
            0 ~ 0.3
        \end{bmatrix}^\top
\end{align*}
\fi
\begin{equation*}
    \begin{split}
        \vA \! &= \!\!
        \begin{bmatrix}
        0.3 \!&\! 0    \!&\! 0 \\
        0   \!&\! 0.15 \!&\! 0 \\
        0   \!&\! 0    \!&\! 0.12
        \end{bmatrix} \!\!,
        \;
        \vB \!= \!\!
        \begin{bmatrix}
            1   \!&\! 0 \\
            0   \!&\! 1 \\
            0.5 \!&\! 0.4
        \end{bmatrix} \!\!,
        \;
        \vC \!=\!\!
        \begin{bmatrix}
            1 \!&\! 0 \!&\! 0 \\
            0 \!&\! 1 \!&\! 0.3
        \end{bmatrix} \!\!.
    \end{split}
\end{equation*}
We set the noise processes to be Gaussian with variances $\vSigma_w = (0.01)^2 \vI_3$ and $\sigma_z=0.01$.  
To select the exploration and truncation lengths, we perform a grid search to determine constants $c_1,c_2$ in $H=c_1T^{2/3}$ and $L=c_2\log T$, calibrated at $T=1500$. These constants were then fixed and applied across all regret experiments.

We ran the ETC algorithm under this setup, repeating each experiment $20$ times with different noise seeds.  
Regret was computed using the two benchmarks described in \textsection\ref{subsec:SDP-GW} and \textsection\ref{subsec:sign-iter}.  
As shown in Figure~\ref{fig:regret-benchmark}, \texttt{SDP+GW} consistently attains higher benchmark values than \texttt{SignIter}, indicating that the latter rarely approaches the true optimum.  
We therefore adopt the \texttt{SDP+GW} benchmark as the reference for regret.

Figure~\ref{fig:regret} reports regrets measured against this benchmark.  
The \texttt{SDP+GW} regret grows slowly and closely follows the $\tilde{\Ocal}(T^{2/3})$ rate, consistent with our theoretical analysis. 
In contrast, the \texttt{SignIter} regret increases much more rapidly, nearly linearly in $T$, demonstrating that the heuristic commit phase is substantially suboptimal relative to the \texttt{SDP+GW} benchmark.

\subsection{Parameter Estimation} \label{subsec:param-est-plots}

In this section, we show that the parameters are estimated effectively.
We generate random instances with $n=5$, $p=3$, and Schur-stable $\vA$ (scaled from i.i.d.\ $\mathcal N(0,1/n)$ entries), with $\vB,\vC$ similarly and noise levels $\vSigma_w=(0.05)^2 \vI_n$, $\sigma_z=0.05$. We study the estimation error of Markov parameters for two spectral radii, $\rho(\vA)=0.1$ and $\rho(\vA)=0.9$, by repeating each experiment 20 times with different noise seeds.

Figures~\ref{fig:param-est-rho-p1} and \ref{fig:param-est-rho-p9} show that the relative estimation error decreases with exploration length $H$. For small $\rho(\vA)$, shorter truncation lengths $L$ yield smaller errors, while for large $\rho(\vA)$ the trend reverses, reflecting a trade-off between memory of states and number of Markov parameters. We also observe a characteristic double-descent effect~\cite{nakkiran2020optimal} as the regression problem transitions from under- to over-determined. For further discussion and derivations, see~\citet{sattar2025learning} and the appendix.

\subsection{Comparison of Practical Methods} \label{subsec:comparison-practice}

To evaluate the performance of the two approaches, we compare them against the true optimum of the regret-benchmark problem~\eqref{eqn:baseline-regret-form}. Since this optimization is NP-hard in general, we restrict to small instances ($n{=}3$, $p{=}2$, $T{=}5{,\dots,}16$) where brute force is feasible. Both \texttt{SDP+GW} and \texttt{SignIter} are randomized, we vary the number of rounding trials ($r{=}1,10,30$) and repeat each experiment 20 times with different seeds.  

\begin{figure}[ht] 
\centering
\includegraphics[width=0.75\linewidth]{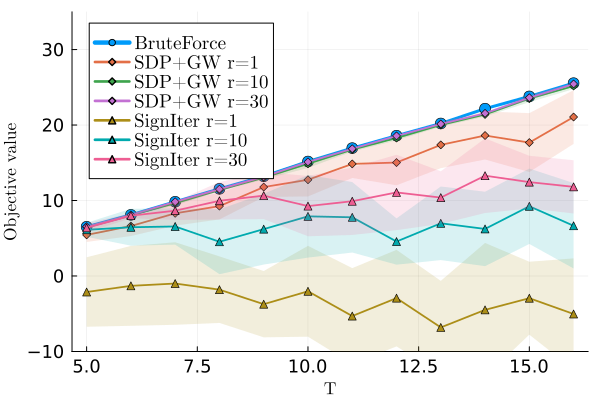}
\caption{Comparison between \texttt{SDP+GW} and \texttt{SignIter} to the true optimum from the brute force method.}
\label{fig:brute-force-comparison}
\end{figure}

Figure~\ref{fig:brute-force-comparison} shows \texttt{SignIter} becomes increasingly suboptimal as $T$ grows, even with more rounding trials. In contrast, \texttt{SDP+GW} consistently returns values very close to the brute-force optimum; even a single trial ($r{=}1$) mostly outperforms all \texttt{SignIter} cases. These results indicate that \texttt{SDP+GW} is an effective and reliable method for the commit-phase optimization.

\section{Related Work} \label{sec:related-work}

Our formulation is connected to, but fundamentally different from, several classical bandit settings. 
Unlike classic multi-armed bandits (MAB), we consider continuous (or combinatorially many) actions and temporally correlated rewards.
Many settings which generalize MAB to nonstationary rewards have been investigated, including restless bandits~\cite{whittle1988restless},
rotting bandits~\cite{levine2017rotting,seznec2019rotting}, recharging bandits~\cite{kleinberg2018recharging}, blocking bandits~\cite{basu2019blocking}, decaying/rebounding bandits~\cite{heidari2016tight,leqi2021rebounding}, delay-dependent rewards \cite{cella2020stochastic}, and bandits with underlying state~\cite{khosravi2023bandits}.
Most of these works consider finitely many discrete actions which affect rewards in a simple and uncorrelated manner, e.g. repeatedly playing the same action results in decayed or improved reward.

The linear bandit setting~\cite{abbasi2011improved} naturally models correlations between actions, but classically considers a static parameter, equivalent in our setting to a fixed but unknown state.
The nonstationary linear bandit setting~\cite{russac2019weighted} allows for drift, 
but requires bounded variation to provide meaningful regret guarantees.
Recently, \citet{trella2024non} consider a nonstationary linear bandit with auto-regressive structure,
but do not model any impact of past actions on the rewards.
Closely related to our work in motivation, \citet{clerici2024linear} study linear bandits with memory.
Our setting differs in two regards: the state depends on all past actions, not just finitely many, and the actions affect the state linearly, leading to a quadratic dependence in the reward, rather than through sublinear powers of a covariance-like ``memory matrix''.

The structure of our reward function bears resemblance to both linear contextual bandits~\cite{lattimore2020bandit} and bilinear bandits \cite{jun2019bilinear}. 
Unlike these settings, our latent state is unobserved. 
Furthermore, unlike the contexutal setting, the state is affected by past actions; unlike for bilinear bandits, the state is influenced through to dynamics and cannot be chosen directly.

Our setting is also related to model-based reinforcement learning for linear dynamical systems~\cite{dean2018regret,simchowitz2020making,mania2019certainty,lale2020regret,lale2020logarithmic},  with two key differences.
First, classical models assume action-independent observations, whereas our rewards depend bilinearly on actions.
Second, quadratic costs in prior work yield linear optimal policies, whereas our  selection is NP-hard.

\section{Conclusion} \label{sec:conclusion}

We study a nonstationary bandit problem where rewards depend bilinearly on actions and latent states evolving under unknown linear dynamics. 
We propose an explore-then-commit algorithm that combines the system parameter estimation and an open-loop action optimization.
Our analysis shows that the algorithm achieves $\tilde{\Ocal}(T^{2/3})$ regret with high probability. 
To address the NP-hard commit-phase optimization, we employed semidefinite relaxation and a rounding scheme, demonstrating its effectiveness empirically. 

A natural direction for future work is to move beyond open-loop policies toward adaptive strategies. 
The recent work~\cite{sattar2025sub} considers a related structure in control theory, studying an optimal control problem of minimizing quadratic cost under linear dynamics with bilinear observations. 
It shows that the optimal feedback policy is nonlinear in the estimated state, suggesting that extending open-loop analysis to feedback or closed-loop designs is both challenging and an important direction for future research.

\subsubsection*{Acknowledgements}

S.D. was partly supported by NSF CCF 2312774, NSF OAC-2311521, NSF IIS-2442137,  a PCCW Affinito-Stewart Award, a gift to the LinkedIn-Cornell Bowers CIS Strategic Partnership, and an AI2050 Early Career Fellowship program at Schmidt Sciences.
M.F. was supported in part by awards NSF TRIPODS II 2023166, NSF CCF 2007036, NSF CCF 2212261, NSF CCF 2312775, and by an Amazon-UW Hub gift award. 
Y.J. was partially supported by the Knut and Alice Wallenberg Foundation Postdoctoral Scholarship Program at MIT-KAW 2022.0366.

% \newpage
\bibliography{refs}

\section*{Checklist}

\begin{enumerate}

  \item For all models and algorithms presented, check if you include:
  \begin{enumerate}
    \item A clear description of the mathematical setting, assumptions, algorithm, and/or model. [Yes]
    \item An analysis of the properties and complexity (time, space, sample size) of any algorithm. [Yes]
    \item (Optional) Anonymized source code, with specification of all dependencies, including external libraries. [Yes]
  \end{enumerate}

  \item For any theoretical claim, check if you include:
  \begin{enumerate}
    \item Statements of the full set of assumptions of all theoretical results. [Yes]
    \item Complete proofs of all theoretical results. [Yes] 
    \item Clear explanations of any assumptions. [Yes]     
  \end{enumerate}

  \item For all figures and tables that present empirical results, check if you include:
  \begin{enumerate}
    \item The code, data, and instructions needed to reproduce the main experimental results (either in the supplemental material or as a URL). [Yes] 
    \item All the training details (e.g., data splits, hyperparameters, how they were chosen). [Yes]
    \item A clear definition of the specific measure or statistics and error bars (e.g., with respect to the random seed after running experiments multiple times). [Yes]
    \item A description of the computing infrastructure used. (e.g., type of GPUs, internal cluster, or cloud provider). [Not Applicable]
  \end{enumerate}

  \item If you are using existing assets (e.g., code, data, models) or curating/releasing new assets, check if you include:
  \begin{enumerate}
    \item Citations of the creator If your work uses existing assets. [Not Applicable]
    \item The license information of the assets, if applicable. [Not Applicable]
    \item New assets either in the supplemental material or as a URL, if applicable. [Not Applicable]
    \item Information about consent from data providers/curators. [Not Applicable]
    \item Discussion of sensible content if applicable, e.g., personally identifiable information or offensive content. [Not Applicable]
  \end{enumerate}

  \item If you used crowdsourcing or conducted research with human subjects, check if you include:
  \begin{enumerate}
    \item The full text of instructions given to participants and screenshots. [Not Applicable]
    \item Descriptions of potential participant risks, with links to Institutional Review Board (IRB) approvals if applicable. [Not Applicable]
    \item The estimated hourly wage paid to participants and the total amount spent on participant compensation. [Not Applicable]
  \end{enumerate}

\end{enumerate}

\clearpage
\appendix
\thispagestyle{empty}

\onecolumn

\input{supplement}

\end{document}

%% file: notation.tex
\usepackage{xcolor}

\definecolor{green}{rgb}{0,0.5,0}

\newcommand{\blue}[1]{\textcolor{blue}{#1}}

\usepackage{amsmath, amsfonts, amssymb, amsthm}
\usepackage{graphicx, subcaption}
\usepackage[normalem]{ulem}
\usepackage{enumitem}
\allowdisplaybreaks 
\usepackage{mdframed}

\newtheorem{theorem}{Theorem}
\newtheorem{assumption}{Assumption}
\newtheorem{proposition}[theorem]{Proposition}
\newtheorem{corollary}{Corollary}
\newtheorem{lemma}{Lemma}
\newtheorem{claim}{Claim}
\newtheorem{definition}{Definition}

\DeclareMathOperator*{\argmax}{arg\,max}
\DeclareMathOperator*{\argmin}{arg\,min}

\usepackage{algorithm}
\usepackage{algpseudocode}

\newcommand{\bv}[1]{{\boldsymbol{#1}}}	    % Bold variable for English letters
\newcommand{\bvgrk}[1]{{\boldsymbol{#1}}}   % Bold variable for Greek letters

\newcommand{\R}{\mathbb{R}}
\newcommand{\Z}{\mathbb{Z}}
\newcommand{\E}{\mathbb{E}}

\newcommand{\norm}[1]{\left\|#1\right\|}

\newcommand{\fronorm}[1]{\left\|#1\right\|_{F}}

\newcommand{\twonorm}[1]{\left\|#1\right\|_{2}}
\newcommand{\infnorm}[1]{\left\|#1\right\|_{\infty}}

\newcommand{\vvec}{{\rm \bf vec}}
\newcommand{\unif}{{\rm Unif}}
\newcommand{\vSigma}{\bvgrk{\Sigma}}
\newcommand{\vDelta}{\bvgrk{\Delta}}
\newcommand{\vtheta}{\bvgrk{\theta}}
\newcommand{\vphi}{\bvgrk{\phi}}
\newcommand{\veta}{\bvgrk{\eta}}
\newcommand{\tr}{{\rm \bf tr}}

\newcommand{\Fcal}{\mathcal{F}}

\newcommand{\Ical}{\mathcal{I}}
\newcommand{\Jcal}{\mathcal{J}}

\newcommand{\Ncal}{\mathcal{N}}
\newcommand{\Ocal}{\mathcal{O}}

\newcommand{\Ucal}{\mathcal{U}}

\newcommand{\vzero}{\bv{0}}

\newcommand{\vb}{\bv{b}}

\newcommand{\vq}{\bv{q}}
\newcommand{\vr}{\bv{r}}
\newcommand{\vs}{\bv{s}}

\newcommand{\vu}{\bv{u}}
\newcommand{\vv}{\bv{v}}
\newcommand{\vw}{\bv{w}}
\newcommand{\vx}{\bv{x}}

\newcommand{\vz}{\bv{z}}

\newcommand{\vutil}{\tilde{\vu}}

\newcommand{\vubar}{\bar{\vu}}

\newcommand{\vwbar}{\bar{\vw}}

\newcommand{\vA}{\bv{A}}
\newcommand{\vB}{\bv{B}}
\newcommand{\vC}{\bv{C}}

\newcommand{\vF}{\bv{F}}
\newcommand{\vG}{\bv{G}}

\newcommand{\vI}{\bv{I}}

\newcommand{\vM}{\bv{M}}

\newcommand{\vR}{\bv{R}}
\newcommand{\vS}{\bv{S}}

\newcommand{\vU}{\bv{U}}
\newcommand{\vV}{\bv{V}}
\newcommand{\vW}{\bv{W}}
\newcommand{\vX}{\bv{X}}

\newcommand{\vGhat}{\hat{\bv{G}}}

\newcommand{\vMhat}{\hat{\bv{M}}}

\newcommand{\vShat}{\hat{\bv{S}}}

\newcommand{\vStil}{\tilde{\bv{S}}}

\newcommand{\vUtil}{\tilde{\bv{U}}}

\newcommand{\distas}{\overset{\text{i.i.d.}}{\sim}}

\definecolor{darkred}{RGB}{150,0,0}
\definecolor{darkgreen}{RGB}{0,150,0}
\definecolor{darkblue}{RGB}{0,0,200}
\usepackage{hyperref}
\hypersetup{colorlinks=true, linkcolor=darkblue, citecolor=darkblue,urlcolor=darkblue}

\newcommand{\diag}{\textup{\textbf{diag}}} 		% Diag
\newcommand{\leqsym}[1]{\stackrel{\text{(#1)}}{\leq}}

\newcommand{\eqsym}[1]{\stackrel{\text{(#1)}}{=}}

%% file: graphical_model.tex
\begin{figure}[t]
\centering
\begin{tikzpicture}[
    node distance=1.2cm and 1.5cm,
    latent/.style={circle, draw, minimum size=0.8cm},
    obs/.style={circle, draw, fill=gray!30, minimum size=0.8cm},
    action/.style={circle, draw, minimum size=0.8cm}
]
% Nodes for t=1
\node[latent] (z1) {$\vx_1$};
\node[obs, below of=z1] (r1) {$r_1$};
\node[action, below of=r1] (a1) {$\vu_1$};
% Nodes for t=2
\node[latent, right=of z1] (ztk) {$\vx_{2}$};
\node[obs, below of=ztk] (rtk) {$r_{2}$};
\node[action, below of=rtk] (atk) {$\vu_{2}$};
% Nodes for t=3
\node[latent, right=of ztk] (zt1) {$\vx_3$};
\node[obs, below of=zt1] (rt1) {$r_{3}$};
\node[action, below of=rt1] (at1) {$\vu_{3}$};
% Nodes for t=t
\node[latent, right=of zt1] (zt) {$\vx_T$};
\node[obs, below of=zt] (rt) {$r_T$};
\node[action, below of=rt] (at) {$\vu_T$};
% Horizontal arrows between latent states
\draw[->] (z1) -- (ztk);
\draw[->] (ztk) -- (zt1);
\draw[->, densely dotted, line width=0.8pt] (zt1) -- (zt);
% Arrows from latent to observed
\draw[->] (z1) -- (r1);
\draw[->] (ztk) -- (rtk);
\draw[->] (zt1) -- (rt1);
\draw[->] (zt) -- (rt);
% Arrows from action to observed
\draw[->] (a1) -- (r1);
\draw[->] (atk) -- (rtk);
\draw[->] (at1) -- (rt1);
\draw[->] (at) -- (rt);
% Arrows from action to next latent state
\draw[->] (a1) -- (ztk);
\draw[->] (atk) -- (zt1);
\draw[->, densely dotted, line width=0.8pt] (at1) -- (zt);

\path (at1) -- (at) node[midway] {$\cdots$};
\end{tikzpicture}
\caption{Graphical Model for Non-Stationary Bandits with controlled Latent Dynamics.}
\label{fig:graphical_model}
\end{figure}

%% file: supplement.tex
\onecolumn
\aistatstitle{Appendix: Explore-then-Commit for Nonstationary Linear Bandits with Latent Dynamics}

\section{Problem Formulation} \label{supp:sec:problem-formulation}
In Section~\ref{sec:problem-formulation}, we formulated the problem concretely, and we stated that it suffices to optimize the objective over the vertices of the hypercube in Section~\ref{subsec:action-set-NP-hard}.
We give a proof of Proposition~\ref{prop:max-vertex} in Section~\ref{supp:subsec:pf-of-prop2}.
Moreover, in Section~\ref{subsec:commit-phase}, we introduced a theoretical objective to the commit phase.
We derive the theoretical commit-phase objective in detail in Section~\ref{supp:subsec:derivation-commit-obj}.

\subsection{Proof of Proposition \ref{prop:max-vertex}} \label{supp:subsec:pf-of-prop2}

\setcounter{theorem}{1}
\begin{proposition} 
    If $\vA \in \R^{n\times n}$ is symmetric with nonnegative diagonal entries, a maximizer of $\vz^\top \vA \vz$ over $[-1,1]^n$ exists at a vertex $\vs \in \{-1,+1\}^n$, that is,
    \begin{align}
        \max_{\|\vz\|_\infty\le1} \vz^\top\vA\vz \;= \max_{\vs \in \{-1,+1\}^n} \vs^\top\vA\vs. \tag{9}
    \end{align}
\end{proposition}

\begin{proof}
    Let $a_{ij}$ be the $(i,j)$-th entry of $\vA$. 
    Fix $\vx=(x_1,\dots,x_n) \in \R^n$ with $\|\vx\|_\infty \le 1$, and choose an arbitrary index $k$.
    Define $\phi_k(x_k) = \vx^\top \vA \vx$ as a function of $x_k \in [-1,1]$, that is,
    \begin{align*}
        \phi_k(x_k) = \vx^\top\vA\vx = a_{kk} x_k^2 + 2b_kx_k + C(x_{-k})
    \end{align*}
    where we denote 
    $b_k= \sum_{i \neq k} a_{ik}x_i$ and $C(x_{-k}) = 2 \sum_{\substack{i<j \\ i,j \neq k}}a_{ij}x_ix_j + \sum_{i \neq k} a_{ii} x_i^2$.
    Since $a_{kk} \ge 0$, $\phi_k(x_k)$ is convex attaining its maximum at an endpoint.
    Thus, replacing $x_k$ by whichever of $\pm1$ gives the larger value (monotonically) increases $\vx^\top\vA\vx$.
    Repeat this coordinatewise for all indices.
    After $n$ steps, we arrive at a vertex $\vs \in \{-1,+1\}^n$ with $\vs^\top\vA\vs \ge \vx^\top \vA\vx$.
    Since $\vx$ was arbitrarily feasible, we have
    \begin{align*}
        \max_{\|\vx\|_{\infty}\le 1} \vx^\top \vA\vx = \max_{\vs \in \{-1,+1\}^n} \vs^\top \vA \vs
    \end{align*}
    which completes the proof.
\end{proof}

\subsection{Derivation of Theoretical Commit-Phase Objective} \label{supp:subsec:derivation-commit-obj}

In the commit phase of the explore-then-commit algorithm, the goal is to maximize the expected cumulative reward from $t=H+1$ to $t=T$.
In Section~\ref{subsec:commit-phase}, we derived that the theoretical objective of the commit phase is
\begin{align*}
    \E\left[ \sum_{t=H+1}^T r_t \right] = \frac{1}{2}\vu_{H+1:T}\vS_{T-H-1} \vu_{H+1:T}
\end{align*}
subject to $\vu_{H+1:T} \in \{-1,+1\}^{p(T-H)}$. 
In this section, we formally derive that the objective is valid.

After the exploration phase, the commit phase starts with the state $\vx_{H+1}$, which is different from the initial state $\vx_0 = \vzero_n$.
Hence, the conditional expectation of cumulative reward will be  
\begin{align*}
    \E\left[ \left. \sum_{t=H+1}^T r_t \, \right| \, \vx_{H+1} \right] &= \sum_{t=H+1}^T \vu_t^\top \vC\vA^{t-H-1} \vx_{H+1} \\
    &\qquad + 
    \begin{bmatrix}
        \vu_{T} \\
        \vu_{T-1} \\
        \vdots \\
        \vdots \\
        \vu_{H+2} \\
        \vu_{H+1}
    \end{bmatrix}^\top 
    \begin{bmatrix}
        \vzero_{p\times p} & \vC\vB & \vC\vA\vB &  \cdots & \cdots & \vC\vA^{T-H-2} \vB \\
        \vzero_{p\times p} & \vzero_{p\times p} & \vC\vB & \cdots & \cdots & \vC\vA^{T-H-3} \vB \\
        \vdots & \ddots & \ddots & \ddots &  & \vdots \\
        \vdots &  & \ddots & \vzero_{p\times p} & \vC\vB & \vC\vA\vB \\
        \vdots &  &  & \ddots & \vzero_{p\times p} & \vC\vB  \\
        \vzero_{p\times p} & \vzero_{p\times p} & \vzero_{p\times p} & \cdots & \vzero_{p\times p} & \vzero_{p\times p}
    \end{bmatrix}
    \begin{bmatrix}
        \vu_{T} \\
        \vu_{T-1} \\
        \vdots \\
        \vdots \\
        \vu_{H+2} \\
        \vu_{H+1}
    \end{bmatrix} \\
    &= \sum_{t=H+1}^T \vu_t^\top \vC\vA^{t-H-1} \vx_{H+1} + \frac{1}{2}\vu_{H+1:T}^\top \vS_{T-H-1}\vu_{H+1:T}
\end{align*}
On the other hand, the state $\vx_{H+1}$ is $\vzero_n$ in expectation because we assumed that $\vx_0=\vzero_n$, $\E[\vu_t]=\vzero_p$, and $\E[\vw_t]=\vzero_n$ for $t=0,\dots,H$.
Therefore, applying the law of total expectation gives the desired objective~\eqref{eqn:commit-theoretical-objective}.

\section{Parameter Estimation} \label{sec:supp:parameter-estimation}

In this section, we discuss the details in parameter estimation. In Section~\ref{supp:subsec:persistence-of-excitation}, we explain what persistence of excitation is and how it is related to parameter estimation. Across the following sections, we derive the estimation error of Markov parameters.

\subsection{Persistence of Excitation} \label{supp:subsec:persistence-of-excitation}

\begin{definition}[Persistence of Excitation \cite{kumar2015stochastic}]
    The action sequence $\{\vu_t\}$ is said to be \emph{persistently exciting} if there is a positive definite matrix $\vU$ such that for all large $n$,
\begin{align*}
    \frac{1}{n}\sum_{k=0}^{n-1} \vphi_k\vphi_k^\top \succeq \vU
\end{align*}
where $\vphi_k$ is a regressor vector constructed from recent inputs $\vu_t$.
\end{definition}

Intuitively, persistent excitation in control theory means that action signal provides sufficient richness of the system's dynamics so that its parameters can be uniquely identified from input-output data.
More importantly, in quantitative aspect, the minimum eigenvalue of $\sum_{k=0}^{n-1} \vphi_k\vphi_k^\top$ grows linearly with $n$, continuously ``exciting'' the system over time at a steady rate. This plays a crucial role in finite sample statistical guarantees, which we will talk about later.

Now we recall the least squares problem that we discussed in Section~\ref{sec:sys-id}:
\begin{align*}
    \vGhat = \argmin_{\vG \in \R^{p \times pL}}  \sum_{t=L+1}^H 
    \left(r_t {-} \vvec(\vG)^\top \vutil_t \right)^2
\end{align*}
where $\vutil_t = \vubar_{t-1} \otimes \vu_t$ and $\vubar_{t-1} = \left[\vu_{t-1}^\top~~\vu_{t-2}^\top~~\cdots~~\vu_{t-L}^\top\right]^\top$.
The solution can be found by solving the normal equation which is
\begin{align*}
    \vvec(\vGhat) = \left(\sum_{t=L+1}^H \vutil_t\vutil_t^\top \right)^{\dagger} \sum_{t=L+1}^H \vutil_t r_t.
\end{align*}
If our action sequence is persistently exciting, then the matrix $\vUtil = \sum_{t=L+1}^H \vutil_t\vutil_t^\top$ is invertible (since positive definite), implying that the parameter $\vG$ is uniquely determined.
Now we show that random Rademacher action sequence is indeed persistently exciting using a result in the recent work~\cite{sattar2025learning}. We state the result for completeness, and the notations are changed to be consistent to our paper.

\begin{mdframed}
    \begin{assumption}[Action Properties \cite{sattar2025learning}] \label{supp:assumption-pe}
        $\vu_0,\dots,\vu_H$ are i.i.d. stochastic isotropic vectors with zero mean and satisfy at least one of the following:
        \begin{enumerate}[leftmargin=*,noitemsep,topsep=0pt,label=(\alph*)]
            \item There exists a scalar $\beta>0$ such that $\twonorm{\vu_t}\leq \beta$ for all $t=0,\dots,H$.
            \item There exists a scalar $m_4>0$ such that $\sup_{\twonorm{\vv}=1} \E[(\vv^\top \vutil_t)^4] \leq m_4$ for all $t=0,\dots,H$ where $\vutil_t$ is as defined in Section~\ref{sec:sys-id}.
        \end{enumerate}
    \end{assumption}
    
    \begin{theorem}[Persistence of Excitation \cite{sattar2025learning}] \label{supp:thm:pe}
        Consider an i.i.d. sequence of random actions $\vu_0$, $\dots$, $\vu_H$ with number of samples satisfying
        \begin{align*}
            H-L \gtrsim \gamma_1(L+1) \left( \log \left( \frac{2(L+1)}{\delta}\right) + \gamma_2p^2L \right).
        \end{align*}
        Suppose either of the following two conditions hold: {\bf (1)} $\{\vu_t\}_{t=0}^H$ satisfies Assumption~\ref{supp:assumption-pe}(a), $\gamma_1=\beta^4L$, and $\gamma_2 = 1$; {\bf (2)} $\{\vu_t\}_{t=0}^H$ satisfies Assumption~\ref{supp:assumption-pe}(b), $\gamma_1 = m_4$, and $\gamma_2=\log(1+16p^2L/\delta)$. Then, with probability at least $1-\delta$, we have
        \begin{align*}
            \lambda_{\min}(\vUtil) \ge (H-L)/4.
        \end{align*}
    \end{theorem}
\end{mdframed}

Theorem~\ref{supp:thm:pe} shows that action distributions that satisfy at least one of the properties in Assumption~\ref{supp:assumption-pe} are persistently exciting with high probability. 
Moreover, the theorem implies that action distribution satisfying Assumption~\ref{supp:assumption-pe}(b) gives a better sample complexity condition.
Observe that the Rademacher distribution $\unif\{-1,+1\}^p$ is centered and isotropic.
It is clear that $\vu_t \sim \unif\{-1,+1\}p$ is bounded with $\beta=\sqrt{p}$.
Moreover, in Section~\ref{supp:subsec:bounded-fourth-moment}, we show that the covariates $\vutil_t$ constructed from Rademacher random vectors have bounded fourth moments $m_4=9$.
Hence, Theorem~\ref{supp:thm:pe} implies that the number of samples required for persistence of excitation is
\begin{align*}
    H-L \gtrsim (L+1) \left( \log \left( \frac{2(L+1)}{\delta}\right) + p^2L\log\left(1+ \frac{16p^2L}{\delta} \right) \right)
\end{align*}
given $\delta \in (0,1)$.

On the other hand, we improve the result by exploiting that our covariates are bounded almost surely. The improvement comes from modifying the proof of Proposition 6.5 in \citet{sattar2024learning} as follows: (1) since  $\vutil_t$ constructed from Rademacher random vectors are bounded almost surely, we have $\|\vUtil\| \leq (H-L)p^2L$ almost surely. Hence, the event $\mathcal{E}$ in the proof is modified as $\mathcal{E} := \{\|\vUtil\| \leq 2p^2L(H-L)\}$; (2) because of the modified event $\mathcal{E}$, we also change the value of $\epsilon$ in the covering argument to $\epsilon = 1/(8p^2L)$. The rest of the proof follows similarly.

Following the proof with the argument above, we can obtain an improved sample complexity for persistent excitation as 
\begin{align} \label{supp:eqn:sample-complexity}
    H-L \ge 32(L+1)m_4 \left( \log \left( \frac{2(L+1)}{\delta}\right) + p^2L \log (1+16p^2L)\right) =: N(p,L,\delta)
\end{align}
where $m_4=9$, an upperbound of the fourth moment of $\vutil_t$.
For future reference, define $N(p,L,\delta)$ to be the right-hand side of \eqref{supp:eqn:sample-complexity}.

\setcounter{theorem}{3}
\subsection{Proof of Theorem 4}

\begin{theorem}
    Under Assumption~\ref{assump:sysID}, let $\{(\vu_t,r_t)\}_{t=0}^H$ be a single trajectory of action-reward pairs collected from the system~\eqref{eqn:sys}.
    For given $\delta \in (0,1)$, suppose 
    \begin{align*}
        H-L \gtrsim (L+1) \left( p^2L \log(p^2L) + \log \left(\frac{L+1}{\delta} \right)\right).
    \end{align*}
    With probability at least $1-\delta$, we have $  $
    \begin{align*}
         \|\vGhat - \vG\|_F \lesssim
         & \frac{\phi(\vA,\rho) \rho^L \|\vC\|}{1-\rho} \sqrt{\frac{p\lambda_{\max}(\vB\vB^\top {+} \vSigma_w)\log(1/\delta)}{H{-}L}}
         + (\sigma_z {+} \sqrt{p}~\Lambda_w^{(L)}) \sqrt{\frac{p^2L\log (p^2L(H{-}L)) {+} \log(L/\delta)}{H{-}L}}
    \end{align*}
    for $\Lambda_w^{(L)} {:=} \sum_{k=0}^{L-1}\sqrt{\lambda_{\max}(\vC\vA^k \vSigma_w (\vA^k)^\top \vC^\top)}$.
\end{theorem}

We use the same notations used in Section~\ref{sec:sys-id}. From \eqref{eqn:estimation_error_form},
we analyze the error by decomposing it as follows:
\begin{align*}
    \|\vG-\vGhat\|_F = \twonorm{\vUtil^{-1} \sum_{t=L+1}^H \vutil_t \zeta_t} \leq \epsilon_1 + \epsilon_2
\end{align*}
where
\begin{align*}
    \epsilon_1 &= \twonorm{\vUtil^{-1} \sum_{t=L+1}^H \vutil_t (\vu_t^\top \vC\vA^L \vx_{t-L})} 
    \quad \text{and} \quad 
    \epsilon_2 = \twonorm{\vUtil^{-1} \sum_{t=L+1}^H \vutil_t(\vu_t^\top \vF\vwbar_{t-1}+z_t)}.
\end{align*}
The error $\epsilon_1$ and $\epsilon_2$ are about truncation bias and multiplier process, respectively.
In Section~\ref{supp:subsec:analysis-truncation-bias}, we derive an upper bound of $\epsilon_1$, and in Section~\ref{supp:subsec:analysis-multiplier-process}, we derive an upper bound of $\epsilon_2$. 
In Section~\ref{supp:subsec:final-parameter-estimation-error}, we give a final error bound for parameter estimation.

\subsection{Analysis of Truncation Bias} \label{supp:subsec:analysis-truncation-bias}

In this section, we analyze the error related to the truncation bias:
\begin{align*}
    \epsilon_1 &= \twonorm{\vUtil^{-1} \sum_{t=L+1}^H \vutil_t (\vu_t^\top \vC\vA^L \vx_{t-L})}.
\end{align*}
Throughout the analysis of $\epsilon_1$, we assume the sample complexity \eqref{supp:eqn:sample-complexity} which leads the persistence of excitation with high probability.
More specifically, given $\delta \in (0,1)$, we assume $H-L \ge N(p,L,\delta/4)$ so that 
\begin{align*}
    \lambda_{\min}(\vUtil) \ge \frac{H-L}{4}
\end{align*}
with probability $1-\frac{\delta}{4}$.
The error $\epsilon_1$ has an upperbound
\begin{align*}
    \epsilon_1 \leq \frac{1}{\lambda_{\min}(\vUtil)} \twonorm{\sum_{t=L+1}^H \vutil_t (\vu_t^\top \vC\vA^L \vx_{t-L})}
\end{align*}
by the consistency of norms.
Given $t \ge 1$, $\vx_t$ can be expressed in terms of previous actions and noises as
\begin{align*}
    \vx_{t} 
    = \sum_{i=1}^{t} \vA^{t-i} \left(\vB \vu_{i-1} + \vw_{i-1}\right)
\end{align*}
when $\vx_0=\vzero_n$.
Observing that $\vutil_t \in \{-1,+1\}^{p^2L}$, we have
\begin{align*}
	\twonorm{\sum_{t=L+1}^H \vutil_t (\vu_t^\top \vC\vA^L \vx_{t-L})}
	&= \sup_{\vtheta \in S^{p^2L-1}} \vtheta^\top \left(\sum_{t=L+1}^H \vutil_t (\vu_t^\top \vC\vA^L \vx_{t-L})\right) \\
	&=   \sum_{t=L+1}^H \sum_{i=1}^{t-L} \left(\sup_{\vtheta \in S^{p^2L-1}} \vtheta^\top\vutil_t\right) \left(\vu_t^\top \vC\vA^{h-i} (\vB\vu_{i-1} + \vw_{i-1})\right) \\
	&= \sum_{i=1}^{H-L} \sum_{t=L+i}^H \left(\sup_{\vtheta \in S^{p^2L-1}}\vtheta^\top \vutil_t \right) \left(\vu_t^\top \vC\vA^{t-i} (\vB\vu_{i-1} + \vw_{i-1})\right) \\
    &\leq \sqrt{p^2L} \left| \sum_{i=1}^{H-L} \left( \sum_{t=L+i}^H \vu_t^\top \vC\vA^{t-i} \right)(\vB\vu_{i-1}+\vw_{i-1})\right|
\end{align*}
Now we apply the following lemma which is a version of Freedman's inequality (adapted from \citet{sattar2025finite}). 
\begin{mdframed}
    \begin{lemma} \label{lem:freedman}
    	Let $(\Fcal_t)_{t\ge 0}$ be a filtration. Let $(\veta_t)_{t \ge 1}$ is a sequence of zero-mean, $\sigma^2$-sub-Gaussian random vectors taking values in $\R^d$, such that $\veta_t$ is $\Fcal_t$-measurable for all $t \ge 1$. Let $(\vx_t)_{t \ge 1}$ be a sequence of random vectors taking values in $\R^d$ such that for all $t \ge 1$, $\vx_t$ is $\Fcal_{t-1}$-measurable and $\|\vx_t\|_2 \le K$ almost surely for some $K>0$. Then for all $\delta \in (0,1)$, $T \ge 1$,
    	\begin{align*}
    		\Pr\left( \left| \sum_{t=1}^T \vx_t^\top \veta_t\right| \leq \sigma K \sqrt{2T\log(2/\delta)}\right) \ge 1-\delta.
    	\end{align*}
    \end{lemma}
\end{mdframed}

First, define a filtration where $\Fcal_0 = \sigma(\vu_{0:H})$ and $\Fcal_i = \sigma(\vu_{0:H}, \vw_{0:i-1})$ for $i \ge 1$.
With the notations in Lemma~\ref{lem:freedman}, take $\vx_i = \left( \sum_{t=L+i}^H \vu_t^\top \vC\vA^{t-i} \right)^\top$ and $\veta_i = \vB\vu_{i-1} + \vw_{i-1}$. 

\begin{claim}
    $\veta_i$ is $\sigma^2$-sub-Gaussian where $\sigma^2=\lambda_{\max}(\vB\vB^\top + \vSigma_w)$.
\end{claim}
\begin{mdframed}
    \begin{proof}
        Let $\vv \in \R^n$ be a unit vector, then $\E[\vv^\top \veta_i] = 0$.
        Given $\lambda \in \R$,
        \begin{align*}
            \E\left[ \exp\left(\lambda (\vv^\top \veta_i)\right)\right] 
            &= \E\left[ \exp\left(\lambda \vv^\top\vB\vu_{i-1}\right)\right] \E\left[ \exp\left(\lambda\vv^\top \vw_{i-1}\right)\right] 
        \end{align*}
        because $\vu_{i-1}$ and $\vw_{i-1}$ are independent.
        First, since $\vw_{i-1}$ is sub-Gaussian with variance proxy $\vSigma_w$, we have $\E\left[ \exp\left(\lambda\vv^\top \vw_{i-1}\right)\right] \leq \exp\left(\frac{1}{2}\lambda^2\vv^\top\vSigma_w\vv\right)$.
        Now we compute $\E\left[ \exp\left(\lambda \vv^\top\vB\vu_{i-1}\right)\right]$. 
        Let $\vb = \vB^\top \vv$, and let $b_k$ and $(\vu_{t-1})_k$ be $k$th coordinate entry of $\vb$ and $\vu_{t-1}$, respectively. 
        Then
        \begin{align*}
            \E\left[ \exp\left(\lambda \vv^\top\vB\vu_{i-1}\right)\right]
            &= \E \left[ \exp\left( \lambda \sum_{k=1}^p b_k (\vu_{i-1})_k \right) \right] \\
            &= \prod_{k=1}^p \E \left[ \exp \left( \lambda b_k (\vu_{i-1})_k \right)\right] \tag{$\because$ each coordinate of $\vu_{i-1}$ is independent} \\
            &= \prod_{k=1}^p \frac{ e^{\lambda b_k} + e^{-\lambda b_k}}{2} \\
            &\leq \prod_{k=1}^p \exp\left( \frac{\lambda^2b_k^2}{2} \right) \\
            &=\exp \left( \frac{\lambda^2}{2} \sum_{k=1}^p b_k^2\right) \\
            &= \exp\left( \frac{1}{2} \lambda^2 \vv^\top \vB\vB^\top \vv\right)
        \end{align*}
        Hence, since $\vv\in\R^n$ is a unit vector, we have
        \begin{align*}
            \E\left[ \exp\left(\lambda (\vv^\top \veta_i)\right)\right] 
            \leq \exp \left( \frac{\lambda^2}{2}\vv^\top \left( \vB\vB^\top +\vSigma_w\right) \vv\right) 
            \leq \exp\left( \frac{\lambda^2}{2}\lambda_{\max}(\vB\vB^\top + \vSigma_w)\right)
        \end{align*}
        which completes the proof.
    \end{proof}
\end{mdframed}

\begin{claim}
    Given $\rho \in (\rho(\vA),1)$,
    $\twonorm{\vx_i} \leq K$ with $K = \sqrt{p}\|\vC\| \frac{\phi(\vA,\rho)\rho^L}{1-\rho}$.
\end{claim}
\begin{mdframed}
    \begin{proof}
        We use Corollary~\ref{cor:lem:finite-gelfands} in order to bound $\twonorm{\vx_i}$ as follows:
        \begin{align*}
        	\|\vx_i\|_2 &\leq \sum_{t=L+i}^H \|\vA^{t-i}\|\|\vC\|\twonorm{\vu_h}
            \leq \sqrt{p} \|\vC\| \sum_{t=L+i}^H \phi(\vA,\rho) \rho^{t-i} 
            \leq \sqrt{p} \|\vC\| \frac{\phi(\vA,\rho)\rho^L}{1-\rho}
        \end{align*}
    \end{proof}
\end{mdframed}

According to Lemma~\ref{lem:freedman}, given $\delta \in (0,1)$, with probability at least $1-\frac{\delta}{4}$, 
\begin{align*}
    \left| \sum_{i=1}^{H-L} \left( \sum_{t=L+i}^H \vu_t^\top \vC\vA^{t-i} \right)(\vB\vu_{i-1}+\vw_{i-1})\right| &=
	\left|\sum_{i=1}^{H-L} \vx_i^\top \veta_i \right| \\
	&\leq \frac{\phi(\vA,\rho)\rho^L\|\vC\|}{1-\rho}\sqrt{2p(H-L)\lambda_{\max}(\vB\vB^\top + \vSigma_w) \log(8/\delta)} .
\end{align*}
Therefore, combining with persistence of excitation, the truncation bias term will be bounded by
\begin{align*}
	\twonorm{\vUtil^{-1} \sum_{t=L+1}^H \vutil_t ( \vu_t^\top \vC\vA^L \vx_{t-L})}
	&\leq \frac{1}{\lambda_{\min}(\vUtil)} \twonorm{ \sum_{t=L+1}^H \vutil_t ( \vu_t^\top \vC\vA^L \vx_{t-L})}\\
    &\leq \frac{\sqrt{p^2L}}{\lambda_{\min}(\vUtil)} \left| \sum_{i=1}^{H-L} \left( \sum_{t=L+i}^H \vu_t^\top \vC\vA^{t-i} \right)(\vB\vu_{i-1}+\vw_{i-1})\right| \\
	&\leq \frac{4p^{3/2}\phi(\vA,\rho)\rho^L\|\vC\|}{1-\rho} \sqrt{\frac{2L\lambda_{\max}(\vB\vB^\top + \vSigma_w) \log(8/\delta)}{H-L}}
\end{align*}
with probability at least $1-\frac{\delta}{2}$.

\subsection{Analysis of Multiplier Process} \label{supp:subsec:analysis-multiplier-process}

In this section, we analyze the error $\epsilon_2$, which is related to multiplier process.
We again first assume that $H-L \geq N(p,L,\delta/4)$ for given $\delta \in (0,1)$ to get the result of persistent excitation, that is,
\begin{align*}
    \lambda_{\min}(\vUtil) \ge \frac{H-L}{4}
\end{align*}
with probability $1-\frac{\delta}{4}$.

The term related to multiplier process has an upper bound:
\begin{align} \label{eqn:multiplier-process-bound1}
    \epsilon_2 = \twonorm{\vUtil^{-1} \sum_{t=L+1}^H \vutil_t ( \vu_t^\top \vF \bar \vw_{t-1}  + z_t) } 
    &\leq
    \frac{\twonorm{\vUtil^{-1/2} \sum_{t=L+1}^H \vutil_{t} ( \vu_t^\top \vF \bar \vw_{t-1}  + z_t)}}{\sqrt{\lambda_{\min}(\vUtil)}}\\
    &\le 
    \underbrace{\frac{1}{\sqrt{\lambda_{\min}(\vUtil)}}\left\| \vUtil^{-1/2} \sum_{t=L+1}^H z_t \vutil_{t} \right\|_2}_{\epsilon_{2,1}} \\
    &\qquad + \underbrace{\frac{1}{\sqrt{\lambda_{\min}(\vUtil)}}\left\| \vUtil^{-1/2} \sum_{t=L+1}^H \vutil_t ( \vu_t^\top \vF \bar \vw_{t-1}) \right\|_2}_{\epsilon_{2,2}} 
\end{align}
To analyze the errors $\epsilon_{2,1}$ and $\epsilon_{2,2}$, we use Theorem~\ref{thm:self-normalized bound} in Section~\ref{supp:subsec:est-misc}.

\subsubsection{Analysis of $\epsilon_{2,1}$} \label{subsubsec:eps_21}

Recall that $(z_t)_{t\ge 0}$ is i.i.d. zero-mean $\sigma^2$-sub-Gaussian.
Define $\Fcal_t = \sigma(\vu_{0:H}, z_{0:t})$.
Then $\vutil_t$ is $\Fcal_{t-1}$-measurable and $z_t$ is $\Fcal_t$-measurable.
Observing that $z_t$ is independent of $\Fcal_{t-1}$, $z_t$ is conditionally $\sigma_z^2$-sub-Gaussian, i.e.,
\begin{align*}
    \E\left[ z_t \mid \Fcal_{t-1} \right]
    = \E[z_t] = 0
	\quad \text{and} \quad 
    \E \left[ e^{\lambda z_t} \mid \Fcal_{t-1} \right]
    = \E\left[ e^{\lambda z_t}  \right] \leq \exp\left(\frac{\lambda^2 \sigma_z^2}{2}\right)
\end{align*}
Let $d = p^2L$. 
We apply the theorem with $\vV = \vI_d$.
Define 
\begin{align*}
    \vV_H = \vI_d + \sum_{t=L+1}^H \vutil_t \vutil_t^\top=\vI_d + \vUtil, \quad \text{and} \quad \vs_H = \sum_{t=L+1}^H z_t \vutil_t.
\end{align*}
Then, for any $\delta>0$, with probability at least $1-\frac{\delta}{4}$,
\begin{align} \label{eqn:multiplier-process-bound2}
    \left\| \vV_H^{-1/2}\vs_H \right\|_2
    = \left\| \left( \vI_d + \vUtil \right)^{-1/2} \sum_{t=L+1}^H z_t \vutil_t \right\|_2    
    \le \sigma_z \sqrt{2\log\left( \frac{4\det(\vV_H)^{1/2}}{\delta}\right)}.
\end{align}
According to Corollary~\ref{cor:multiplier-process-inequality} with $\lambda=1$, we have 
\begin{align*}
    \left\| \vUtil^{-1/2} \sum_{t=L+1}^H z_t \vutil_t \right\|_2 
    \leq \sqrt{1 + \frac{1}{\lambda_{\min}(\vUtil)}}
    \left\| \left( \vI_d + \vUtil \right)^{-1/2} \sum_{t=L+1}^H z_t \vutil_t \right\|_2.
\end{align*}
Combining the inequality with \eqref{eqn:multiplier-process-bound2}, we get
\begin{align} \label{eqn:multiplier-process-bound3}
    \left\| \vUtil^{-1/2} \sum_{t=L+1}^H z_t \vutil_t \right\|_2 
    \leq \sigma_z
    \sqrt{2 \left(1 + \frac{1}{\lambda_{\min}(\vUtil)}\right) \log \left(\frac{4\det(\vV_H)^{1/2}}{\delta} \right)}.
\end{align}
\begin{claim} \label{claim:lam-max-Util}
    $\lambda_{\max}(\vUtil) \leq d(H-L)$
\end{claim}
\begin{mdframed}
    \begin{proof}
        If we show that $\vUtil \preceq d(H-L)\vI_{d}$, then the claim is proved. 
        We can show that $d(H-L)\vI_{d} - \vUtil$ is positive semidefinite by Cauchy-Schwarz inequality:
        \begin{align*}
            \vx^\top \vUtil \vx = \sum_{t=L+1}^H (\vutil_t^\top \vx)^2 \leq \sum_{t=L+1}^H \twonorm{\vutil_t}^2 \twonorm{\vx}^2 = d(H-L) \twonorm{\vx}^2
        \end{align*}
        Since the inequality above holds for all $\vx \in \R^{d}$, the proof is done.
    \end{proof}
\end{mdframed}
By Claim~\ref{claim:lam-max-Util}, 
\begin{align*}
    \det(\vV_H) = \det(\vI_d + \vUtil) \leq \left(1 + \lambda_{\max}(\vUtil)\right)^d \leq \Big( 1 + d(H-L)\Big)^d.
\end{align*}
From \eqref{eqn:multiplier-process-bound3}, we now have
\begin{align}
    \twonorm{\vUtil^{-1/2} \sum_{t=L+1}^H z_t \vutil_t}
    &\leq \sigma_z
    \sqrt{2 \left(1 + \frac{1}{\lambda_{\min}(\vUtil)}\right) \log \frac{4\left(1 + d(H-L)\right)^{d/2}}{\delta} }.
\end{align}
so that
\begin{align*}
    \frac{1}{\sqrt{\lambda_{\min}(\vUtil)}} \twonorm{\vUtil^{-1/2} \sum_{t=L+1}^H z_t \vutil_t} 
    &\leq \frac{\sigma_z\sqrt{1 + \lambda_{\min}(\vUtil)}}{\lambda_{\min}(\vUtil)}
    \sqrt{2\log \frac{4\left(1 + d(H-L)\right)^{d/2}}{\delta} }.
\end{align*}
Now we combine with persistence of excitation. Under the sample complexity requirement, since we have $\lambda_{\min}(\vUtil) \ge (H-L)/4$ with probability at least $1-\frac{\delta}{4}$,
we have
\begin{align*}
	1+\lambda_{\min}(\vUtil) \leq \left( \frac{4}{H-L} +1 \right) \lambda_{\min}(\vUtil).
\end{align*}
If we choose $H$ and $L$ with $H-L \ge 4$, then $2(H-L) \ge H-L+4$, implying
\begin{align*}
    \frac{\sqrt{1 + \lambda_{\min}(\vUtil)}}{\lambda_{\min}(\vUtil)} 
    \leq \sqrt{ \frac{H-L+4}{\lambda_{\min}(\vUtil)(H-L)}}
    \leq \sqrt{ \frac{2}{\lambda_{\min}(\vUtil)}}
    \leq \sqrt{ \frac{8}{H-L}}.
\end{align*}
Moreover, under $H-L \ge 4$, we have $d(H-L) \ge 2$ so that
\begin{align*}
    \log(1+d(H-L))^{d/2} = d\log\sqrt{1+d(H-L)} \leq d \log (d(H-L))
\end{align*}
for $d(H-L) \ge 2$.

By combining all, with probability at least $1-\frac{\delta}{2}$,
\begin{align*}
    \frac{1}{\sqrt{\lambda_{\min}(\vUtil)}} \twonorm{\vUtil^{-1/2} \sum_{t=L+1}^H z_t \vutil_t} 
    &\leq 4\sigma_z\sqrt{\frac{d\log(d(H-L)) + \log(4/\delta)}{H-L}}.
\end{align*}
provided that $H-L \ge N(p,L,\delta/4)$.

\subsubsection{Analysis of $\epsilon_{2,2}$}
We analyze the error $\epsilon_{2,2}$ by decomposing in the following way:
\begin{align*}
    \frac{1}{\sqrt{\lambda_{\min}(\vUtil)}}\left\| \vUtil^{-1/2} \sum_{t=L+1}^H \vutil_t ( \vu_t^\top \vF \bar \vw_{t-1}) \right\|_2
    &= \frac{1}{\sqrt{\lambda_{\min}(\vUtil)}}\left\| \vUtil^{-1/2} \sum_{t=L+1}^H \vutil_t \left(\sum_{k=0}^{L-1} \vu_t^\top\vC\vA^k \vw_{t-k-1}\right)\right\| \\
    &\leq \sum_{k=0}^{L-1} \frac{1}{\sqrt{\lambda_{\min}(\vUtil)}}\left\| \vUtil^{-1/2} \sum_{t=L+1}^H \vutil_t \left( \vu_t^\top\vC\vA^k \vw_{t-k-1}\right)\right\|   
\end{align*}
For each $k=0,1,\dots,L-1$, define a filtration $\mathcal{F}_t^{(k)} = \sigma(\vu_{0:H}, \vw_{0:t-k-1})$ for $t\ge L+1$.
Let $\eta_t^{(k)} = \vu_t^\top\vC\vA^k\vw_{t-k-1}$. Then $\eta_t^{(k)}$ is $\Fcal_t^{(k)}$-measurable.
Observe that $\eta_t^{(k)} \mid \Fcal_{t-1}^{(k)}$ is mean zero sub-Gaussian because
\begin{align*}
    \E \left[ \eta_t^{(k)} \mid \Fcal_{t-1}^{(k)}\right] = \vu_t^\top \vC\vA^k \E [\vw_{t-k-1}] = 0
\end{align*}
and, for any $\lambda \in \R$,
\begin{align*}
    \E\left[\left.\exp\left(\lambda\eta_t^{(k)}\right)\right| \Fcal_{t-1}^{(k)} \right] 
    &= \E \left[\left. \exp \left( \lambda \vu_t^\top \vC\vA^k \vw_{t-k-1}\right)\right| \Fcal_{t-1}^{(k)} \right] \\
    &\leq \exp \left( \frac{\lambda^2}{2} \vu_t^\top \vC\vA^k \vSigma_w (\vA^k)^\top \vC^\top \vu_h\right) \\
    &\leq \exp \left( \frac{\lambda^2}{2} \lambda_{\max}(\vC\vA^k \vSigma_w (\vA^k)^\top \vC^\top)\twonorm{\vu_t}^2 \right) \\
    &= p \lambda_{\max}\left(\vC\vA^k \vSigma_w (\vA^k)^\top \vC^\top\right)
\end{align*}
Let $\vS_k = \vC\vA^k \vSigma_w (\vA^k)^\top \vC^\top$ for brevity.
We repeat a similar argument in Section~\ref{subsubsec:eps_21}.
Fix $k$. For any $\delta \in (0,1)$, with probability at least $1-\left(\frac{\delta}{4L} +\frac{\delta}{4}\right)$, we have
\begin{align*}
    \frac{1}{\sqrt{\lambda_{\min}(\vUtil)}}\left\| \vUtil^{-1/2} \sum_{t=L+1}^H \eta_t^{(k)} \vutil_t \right\| 
    \leq 4\sqrt{p\lambda_{\max}(\vS_k)} \sqrt{\frac{d\log (d(H-L)) + \log(4L/\delta)}{H-L}}
\end{align*}
if $H-L \ge N(p,L,\delta/4)$ (the sample complexity).
Therefore, with probability at least $1-\frac{\delta}{2}$, we have
\begin{align*}
    \frac{1}{\sqrt{\lambda_{\min}(\vUtil)}}\left\| \vUtil^{-1/2} \sum_{t=L+1}^H \vutil_t ( \vu_t^\top \vF \bar \vw_{t-1}) \right\|_2 
    &\leq \sum_{k=0}^{L-1} \frac{1}{\sqrt{\lambda_{\min}(\vUtil)}}\left\| \vUtil^{-1/2} \sum_{t=L+1}^H \vutil_t \left( \vu_t^\top\vC\vA^k \vw_{t-k-1}\right)\right\| \\
    &\leq 4\sqrt{\frac{d\log (d(H-L)) + \log(4L/\delta)}{H-L}} \sum_{k=0}^{L-1} \sqrt{p\,\lambda_{\max}(\vS_k)} 
\end{align*}
if $H-L \ge N(p,L,\delta/4)$.

\subsection{Final Parameter Estimation Error Bound} \label{supp:subsec:final-parameter-estimation-error}

We now combine the errors $\epsilon_1$ and $\epsilon_2$ (and $\epsilon_{2,1}$, $\epsilon_{2,2}$).
Under the assumption of persistence of excitation, we have
\begin{align*}
    \epsilon_1 \le  \frac{4p^{3/2}\phi(\vA,\rho)\rho^L\|\vC\|}{1-\rho} \sqrt{\frac{2L\lambda_{\max}(\vB\vB^\top + \vSigma_w) \log(8/\delta)}{H-L}}
\end{align*}
with probability at least $1-\frac{\delta}{2}$,
\begin{align*}
    \epsilon_2 &\le \epsilon_{2,1} + \epsilon_{2,2} \\
    &\leq 4\sigma_z\sqrt{\frac{p^2L\log(p^2L(H-L)) + \log(4/\delta)}{H-L}} \\
    &\qquad \qquad+ 4\sqrt{\frac{p^2L\log (p^2L(H-L)) + \log(4L/\delta)}{H-L}} \sum_{k=0}^{L-1} \sqrt{p\,\lambda_{\max}(\vC\vA^k\vSigma_w(\vA^k)^\top \vC^\top)} \\
    &\leq 4\left(\sigma_z + \sqrt{p}\sum_{k=0}^{L-1} \sqrt{\lambda_{\max}(\vC\vA^k\vSigma_w(\vA^k)^\top \vC^\top)}\right)\sqrt{\frac{p^2L\log (p^2L(H-L)) + \log(4L/\delta)}{H-L}}
\end{align*}
with probability at least $1-\frac{3\delta}{4}$, where $d=p^2L$ and $\vS_k = \vC\vA^k\vSigma_w(\vA^k)^\top \vC^\top$.
As a result, we have a final bound for multiplier process. With probability at least $1-\delta$, 
\begin{align*}
    \|\vG-\vGhat\|_F
    &\leq \epsilon_1 + \epsilon_2 \\
    &\leq \frac{4p^{3/2}\phi(\vA,\rho)\rho^L\|\vC\|}{1-\rho} \sqrt{\frac{2L\lambda_{\max}(\vB\vB^\top + \vSigma_w) \log(8/\delta)}{H-L}} \\
    &\qquad+ 4\left(\sigma_z + \sqrt{p}\sum_{k=0}^{L-1} \sqrt{\lambda_{\max}(\vC\vA^k\vSigma_w(\vA^k)^\top \vC^\top)}\right)\sqrt{\frac{p^2L\log (p^2L(H-L)) + \log(4L/\delta)}{H-L}}
\end{align*}
provided that $H-L \ge N(p,L,\delta/4)$.

\subsection{Proof of Bounded Fourth Moment of Covariates} \label{supp:subsec:bounded-fourth-moment}

In this section, we prove that our covariate $\vutil_t$ has bounded fourth moment.
Observe that the fourth moment of $\vutil_t \in \R^{p^2L}$ is defined to be
\begin{align*}
    \sup_{\substack{\|\vv\|_2 = 1 \\ \vv \in \R^{p^2L}}} \E\left[ (\vv^\top \vutil_t)^4\right]
\end{align*}
since $\vutil_t$ is centered and isotropic.

Given $\vv \in \R^{p^2L}$ with $\twonorm{\vv}=1$, let $\vV \in \R^{p\times pL}$ such that $\vvec(\vV) = \vv$ (so that $\|\vV\|_F =1$). 
By the property of Kronecker product, we have
\begin{align*}
    \E\left[(\vv^\top \vutil_t)^4 \right] = \E \left[ ((\vubar_{t-1}\otimes \vu_t)^\top\vv)^4\right] = \E \left[ (\vu_t^\top \vV \vubar_{t-1})^4\right].
\end{align*}
Let $\vV_1,\dots,\vV_L$ be the $p\times p$ matrices such that
\begin{align*}
    \vV = \begin{bmatrix}
        \vV_1 & \vV_2 & \cdots & \vV_L
    \end{bmatrix}
\end{align*}
then $\vu_t^\top \vV \vubar_{t-1} = \sum_{k=1}^L \vu_t^\top \vV_k \vu_{t-k}$.
For simplicity, let $z_k = \vu_t^\top \vV_k\vu_{t-k}$. Then, we have
\begin{align*}
    \E \left[ (\vu_t^\top \vV \vubar_{t-1})^4\right] 
    &= \E \left[ \left( \sum_{k=1}^L z_k \right)^4 \right] \\
    &= \sum_{k=1}^L \E[z_k^4] + 4 \sum_{k \neq l} \E[z_k^3z_l] + 6 \sum_{k<l}\E[z_k^2z_l^2] + 12 \sum_{\substack{k\neq l, k\neq m \\ l \neq m}} \E[z_k^2z_lz_m] + 24 \sum_{k<l<m<n} \E[z_kz_lz_mz_n] 
\end{align*}
Since $\vu_t,\vu_{t-1},\dots,\vu_{t-L}$ are independent, we know $\E[z_k]=0$, and moreover, $\E\left[\left. z_k \,\right| \vu_t \right] = 0$.
Then, for $k \neq l$, 
\begin{align*}
    \E[z_kz_l] 
    &= \E\left[ \E \left[ \left. z_k z_l \,\right| \vu_t \right] \right] \tag{law of total expectation}\\
    &= \E \left[ \E \left[ \left. z_k \, \right| \vu_t \right] \E \left[ \left. z_l \, \right| \vu_t \right]\right] \tag{$z_k$ and $z_l$ are conditionally independent.} \\
    &= 0
\end{align*}
Likewise, we can show that 
$\E[z_k^3z_l]=0$ for $k \neq l$, $\E[z_k^2z_lz_m] = 0$ for $k\neq l$, $k \neq m$, and $l\neq m$, and $\E[z_kz_lz_mz_n]=0$ for $k<l<m<n$, showing that
\begin{align}
     \E\left[(\vv^\top \vutil_t)^4 \right] = \E \left[ (\vu_t^\top \vV \vubar_{t-1})^4\right]
    = \sum_{k=1}^L \E[z_k^4] + 6 \sum_{k<l}\E[z_k^2z_l^2]. \label{eqn:forth_moment_expression}
\end{align}
Now we upper bound $\E[z_k^4]$ for $k = 1, \dots, L$, as follows: Let $\vv_{k,1}, \dots, \vv_{k,p}$ denote the columns of $\vV_k$, and $(\vu_{t-k})_i$ denote the $i$-th entry of the Rademacher random vector $\vu_{t-k}$. Then, we have
\begin{align}
    \E[z_k^4] = \E\left[\left(\vu_t^\top \vV_k\vu_{t-k}\right)^4\right] &= \E\left[ \left( \sum_{i=1}^p \vu_t^\top \vv_{k,i} (\vu_{t-k})_i\right)^4 \right], \nonumber \\
    &\eqsym{a} \sum_{i=1}^p \E\left[ \left( \vu_t^\top \vv_{k,i} (\vu_{t-k})_i\right)^4 \right] + 6 \sum_{i<j}\E\left[ \left( \vu_t^\top \vv_{k,i} (\vu_{t-k})_i\right)^2 \left( \vu_t^\top \vv_{k,j} (\vu_{t-k})_j\right)^2 \right], \nonumber \\
    &\eqsym{b} \sum_{i=1}^p \E\left[ \left( \vu_t^\top \vv_{k,i}\right)^4 \right] + 6 \sum_{i<j}\E\left[ \left( \vu_t^\top \vv_{k,i} \right)^2 \left( \vu_t^\top \vv_{k,j} \right)^2 \right], \nonumber \\
    &\leqsym{c} 3\sum_{i=1}^p \twonorm{\vv_{k,i}}^4 + 18 \sum_{i<j} \twonorm{\vv_{k,i}}^2 \twonorm{\vv_{k,j}}^2, \nonumber \\
    & \leq 9 \left( \sum_{i=1}^p \twonorm{\vv_{k,i}}^2\right)^2 
    = 9 \fronorm{\vV_k}^4, \label{eqn:bound_on_zk_four}
\end{align}
where we obtained (a) from using the same arguments as used for simplifying $ \E \left[ (\vu_t^\top \vV \vubar_{t-1})^4\right]$ above, (b) follows from the independence of $(\vu_{t-k})_i$, $(\vu_{t-k})_i$ and $\E[(\vu_{t-k})_i^4] = E[(\vu_{t-k})_i^2(\vu_{t-k})_j^2] = 1$, and lastly (c) is obtained from the fact that, given fixed vetors $\vq, \vq_1,\vq_2$, for any Rademacher random vector $\vu$, we have $\E[(\vq^\top\vu)^4] \leq 3\twonorm{\vq}^4$ and $\E[(\vq_1^\top\vu)^2(\vq_2^\top\vu)^2] \leq 3\twonorm{\vq_1}^2\twonorm{\vq_2}^2$.

Next, we upper bound $\E[z_k^2z_l^2]$ for $k< l$ as follows:  Let $\vv_{k,1}, \dots, \vv_{k,p}$ denote the columns of $\vV_k$, $\vv_{l,1}, \dots, \vv_{l,p}$ denote the columns of $\vV_l$,  and $(\vu_{t-k})_i$, $(\vu_{t-l})_i$ denote the $i$-th entry of the Rademacher random vectors $\vu_{t-k}$, $\vu_{t-l}$ respectively. Then, we have
\begin{align}
    \E[z_k^2 z_l^2] &= \E\left[\left(\vu_t^\top \vV_k\vu_{t-k}\right)^2 \left(\vu_t^\top \vV_l\vu_{t-l}\right)^2\right], \nonumber \\
    &= \E\left[ \left( \sum_{i=1}^p \vu_t^\top \vv_{k,i} (\vu_{t-k})_i\right)^2 \left( \sum_{j=1}^p \vu_t^\top \vv_{l,j} (\vu_{t-l})_j\right)^2 \right], \nonumber \\ 
    & \eqsym{i} \sum_{i=1}^p  \sum_{j=1}^p \E\left[\left(\vu_t^\top \vv_{k,i}\right)^2 \left(\vu_t^\top \vv_{l,j}\right)^2\right], \nonumber \\
    &\leqsym{ii} 3 \sum_{i=1}^p  \sum_{j=1}^p \twonorm{\vv_{k,i}}^2 \twonorm{\vv_{l,j}}^2 = 3 \fronorm{\vV_k}^2 \fronorm{\vV_l}^2, \label{eqn:bound_on_zk_zl_square}
\end{align}
where we get (i) from the independence of $(\vu_{t-k})_i, (\vu_{t-l})_j$ for all $k<l$, and $i,j=1,\dots,p$, and (ii) is obtained from using the same argument as used to get (c) in \eqref{eqn:bound_on_zk_four}. Finally, combining \eqref{eqn:bound_on_zk_four}, and \eqref{eqn:bound_on_zk_zl_square} into \eqref{eqn:forth_moment_expression}, we obtain the following upper bound,
\begin{align}
     \E\left[(\vv^\top \vutil_t)^4 \right] 
    &= \sum_{k=1}^L \E[z_k^4] + 6 \sum_{k<l}\E[z_k^2z_l^2], \nonumber \\
    & \leq  9 \sum_{k=1}^L  \fronorm{\vV_k}^4 + 18 \sum_{k<l} \fronorm{\vV_k}^2 \fronorm{\vV_l}^2. \nonumber \\
    &= 9 \left( \sum_{k=1}^L \fronorm{\vV_k}^2\right)^2 = 9 \fronorm{\vV}^2 = 9.
\end{align}

\subsection{Miscellaneous} \label{supp:subsec:est-misc}

\begin{lemma} \label{lem:finite-gelfand}
    If the matrix norm $\|\cdot\|$ is consistent, the quantity
    \begin{align*}
        \phi(\vA, \rho) = \sup_{k \in \Z_+} \frac{\|\vA^k\|}{\rho^k}
    \end{align*}
    is finite for all $\rho>\rho(\vA)$. 
\end{lemma}
\begin{proof}
    Recall Gelfand's formula:
    for any consistent matrix norm $\|\cdot\|$, we have
    \begin{align} \label{eqn:spectral-radius-limit}
    	\rho(A) = \lim_{k \to \infty} \|\vA^k\|^{1/k}
    \end{align}
    and the limit approaches from above.
    By Gelfand's formula, for all $\rho > \rho(\vA)$, there exists $N$ such that
    \begin{align*}
    	\rho(\vA) \leq \|\vA^k\|^{1/k} < \rho
    \end{align*}
    whenever $k \ge N$.
    Then, we have
    \begin{align*}
        \phi(\vA,\rho) = \max\left\{\max_{k=1,\dots,N-1}\frac{\|\vA^k\|}{\rho^k}, \sup_{k \ge N} \frac{\|\vA^k\|}{\rho^k}\right\}
    \end{align*}
    implying that $\phi(\vA,\rho)$ is finite because $\|\vA^k\|/\rho^k <1$ for $k \ge N$.
\end{proof}
\begin{corollary} \label{cor:lem:finite-gelfands}
    Given $\rho > \rho(\vA)$, for all $k \ge 1$, we have $\|\vA^k\| \le \phi(\vA,\rho) \rho^k$.
\end{corollary}

\begin{proposition}\label{prop:vUtil inequality}
    The following holds for $\lambda >0$:
    \begin{align*}
        \left( \lambda \vI_d + \vUtil \right)^{-1} \preceq \vUtil^{-1} \preceq \left( 1+\frac{\lambda}{\lambda_{\min}(\vUtil)} \right) \left( \lambda \vI_d + \vUtil \right)^{-1}
    \end{align*}
\end{proposition}
\begin{proof}
    Since $\lambda \vI_d + \vUtil \succeq \vUtil$ for $\lambda >0$, we have $\left( \lambda \vI_d + \vUtil \right)^{-1} \preceq \vUtil^{-1}$.
    In order to show the other Loewner order, it is equivalent to show that
    \begin{align*}
        \vUtil \succeq \frac{\lambda_{\min}(\vUtil)}{\lambda_{\min}(\vUtil) + \lambda} \left( \lambda \vI_d + \vUtil \right).
    \end{align*}
    We want to find $\gamma>0$ such that
    $\vUtil - \gamma \left( \lambda \vI_d + \vUtil \right) \succeq 0$, that is, $\lambda_{\min} \left(\vUtil - \gamma \left( \lambda \vI_d + \vUtil \right) \right) \ge 0$.
    \begin{align*}
        \lambda_{\min} \left(\vUtil - \gamma \left( \lambda \vI_d + \vUtil \right) \right) 
        &=\lambda_{\min} \left( (1-\gamma)\vUtil - \gamma \lambda \vI_d \right) \\
        &= (1-\gamma) \lambda_{\min}(\vUtil) - \gamma \lambda \\
        &= \lambda_{\min}(\vUtil) - \gamma\left( \lambda_{\min}(\vUtil) + \lambda \right) \ge 0
    \end{align*}
    Hence, for $0 < \gamma \leq \frac{\lambda_{\min}(\vUtil)}{\lambda_{\min}(\vUtil) + \lambda}$, or $\frac{1}{\gamma} \ge 1 + \frac{\lambda}{\lambda_{\min}(\vUtil)}$, 
    we have 
    \begin{align*}
        \vUtil^{-1} \preceq \frac{1}{\gamma}\left( \lambda \vI_d + \vUtil \right)^{-1}.
    \end{align*}
    This completes the proof.
\end{proof}
\begin{corollary} \label{cor:multiplier-process-inequality}
    For any $\vx \in \R^d$, we have
    \begin{align*}
        \left\| \vUtil^{-1/2} \vx \right\|_2 \leq \sqrt{1+\frac{\lambda}{\lambda_{\min}(\vUtil)}} \left\|\left(\lambda \vI_d + \vUtil \right)^{-1/2} \vx \right\|_2.
    \end{align*}
\end{corollary}
\begin{proof}
    Since $\left( 1+\frac{\lambda}{\lambda_{\min}(\vUtil)} \right) \left( \lambda \vI_d + \vUtil \right)^{-1} - \vUtil^{-1}$ is positive semi-definite by Proposition~\ref{prop:vUtil inequality}, for any $\vx \in \R^d$, we have
    \begin{align*}
        \vx^\top \vUtil^{-1}\vx \leq \left( 1+\frac{\lambda}{\lambda_{\min}(\vUtil)} \right) \vx^\top \left( \lambda \vI_d + \vUtil \right)^{-1} \vx.
    \end{align*}
    which completes the proof.
\end{proof}

\begin{theorem}[Self-Normalized Bound for Vector-Valued Martingales, \cite{abbasi2011improved}] \label{thm:self-normalized bound}
    Let $\{\Fcal_t\}_{t=0}^\infty$ be a filtration. Let $\{\eta_t\}_{t=1}^\infty$ be a real-valued stochastic process such that $\eta_t$ is $\Fcal_t$-measurable and $\eta_t$ is conditionally $R$-sub-Gaussian for some $R \ge 0$, i.e.,
    \begin{align*}
        \forall \lambda \in \R, \qquad \E \left[ e^{\lambda \eta_t} \mid \Fcal_{t-1} \right] \leq \exp \left( \frac{\lambda^2 R^2}{2} \right).
    \end{align*}
    Let $\{\vx_t\}_{t=1}^\infty$ be an $\R^d$-valued stochastic process such that $\vx_t$ is $\Fcal_{t-1}$-measurable. Assume that $\vV$ is a $d \times d$ positive definite matrix. For any $t \ge 0$, define
    \begin{align*}
        \bar{\vV}_t = \vV + \sum_{s=1}^t \vx_s\vx_s^\top \qquad \vs_t = \sum_{s=1}^t \eta_s\vx_s.
    \end{align*}
    Then, for any $\delta>0$, with probability at least $1-\delta$, for all $t \ge 0$,
    \begin{align*}
        \|\bar{\vV}^{-1/2}\vs_t\| \leq R \sqrt{2 \log \left( \frac{\det(\bar{\vV}_t)^{1/2}\det(\vV)^{-1/2}}{\delta}\right)}.
    \end{align*}
\end{theorem}

\section{Regret Analysis}

In this section, we derive the regret in detail and prove Proposition 5 and 6. 
Recall that
\begin{align*}
    &\vu_{0:T}^\star = \argmax_{\vu_{0:T}} \E\left[ \sum_{t=0}^T r_t \right] = \frac{1}{2}\vu_{0:T}^\top \vS_T \vu_{0:T} \quad \text{subject to } \vu_{0:T} \in \{-1,+1\}^{p(T+1)}, \\
    &\vu_{0:H}^\pi \distas \unif(\{-1,+1\}^p), \quad\text{and} \\
    &\vu_{H+1:T}^\pi = \argmax_{\vu_{H+1:T}} \frac{1}{2} \vu_{H+1:T}^\top \vShat_{T-H-1} \vu_{H+1:T}^\top \quad \text{subject to } \vu_{H+1:T} \in \{-1,+1\}^{p(T-H)}.
\end{align*}
where $\pi$ denotes the ETC algorithm policy.  
Let $\{r_t^\star\}_{t=0}^T$ denote the random rewards collected under $\vu_{0:T}^\star$, and $\{r_t^\pi\}_{t=0}^T$ denote the random rewards collected by policy $\pi$. 
Then the regret of $\pi$ is 
\begin{align*}
    R_T(\pi) = \E \left[\sum_{t=0}^T r_t^\star - \sum_{t=0}^T r_t^\pi \right] \quad
    \text{where} \quad
    \E\left[ \sum_{t=0}^T r_t^\star \right] = \frac{1}{2}(\vu_{0:T}^\star)^\top \vS_T \vu_{0:T}.
\end{align*}
For the expected cumulative reward for $\pi$, we split into the two phases. 
For $t=0,\dots,H$, we have $\{\vu_t^\pi\}_{t=0}^\pi \distas \unif(\{-1,+1\}^p)$, and note that the distribution has zero-mean and identity covariance. Since the conditional expectation is
\begin{align*}
    \E\left[\sum_{t=1}^H r_t^\pi \, \bigg| \, \{\vu_t^\pi\}_{t=0}^H\right]
    &= (\vu_{0:H}^\pi)^\top \vM_H \vu_{0:H}^\pi 
    = \tr(\vM_H \vu_{0:H}^\pi (\vu_{0:H}^\pi)^\top),
\end{align*}
the law of total expectation implies that
\begin{align*}
    \E \left[\sum_{t=0}^H r_t^\pi \right] &= \E\left[\E\left[\left.\sum_{t=1}^H y_t^\pi \, \right| \, \{\vu_t^\pi\}_{t=0}^H\right] \right] \\
    &= \E\left[ \tr(\vM_H \vu_{0:H}^\pi (\vu_{0:H}^\pi)^\top) \right] \\
    &= \tr\left(\vM_H \E\left[\vu_{0:H}^\pi (\vu_{0:H}^\pi)^\pi \right] \right) \\
    &= \tr (\vM_H \vI_{p(H+1)}) \\
    &= 0
\end{align*}
Now we consider the remaining steps. 
Let $\vx_{H+1}^\pi$ be the random state after the exploration, then we know that $\E\left[\vx_{H+1}^\pi\right] = \vzero_n$ because $\vx_0=\vzero_n$ and the expectation of noise processes and Rademacher actions are zero.
Since we have
\begin{align*}
    \E\left[\left.\sum_{t=H+1}^T r_t^\pi\right| \vx_{H+1}^\pi\right] 
    &= \sum_{t=H+1}^T (\vu_t^\pi)^\top \vC\vA^{t-H-1} \vx_{H+1}^\pi + \frac{1}{2}(\vu_{H+1:T}^\pi)^\top \vS_{T-H-1}\vu_{H+1:T}^\pi
\end{align*}
the law of total expectation again implies that
\begin{align*}
    \E\left[ \sum_{t=H+1}^T r_t^\pi \right] &= \frac{1}{2}(\vu_{H+1:T}^\pi)^\top \vS_{T-H-1}\vu_{H+1:T}^\pi.
\end{align*}
Therefore, the regret is 
\begin{align*}
    R_T(\pi) = \frac{1}{2}(\vu_{0:T}^\star)^\top \vS_T \vu_{0:T} - \frac{1}{2}(\vu_{H+1:T}^\pi)^\top \vS_{T-H-1}\vu_{H+1:T}^\pi.
\end{align*}
For the regret analysis, let 
\begin{align*}
    \vutil_{H+1:T} = \argmax_{\vu_{H+1:T}} \frac{1}{2}\vu_{H+1:T}^\top \vS_{T-H-1}\vu_{H+1:T} \quad \text{subject to } \quad \vu_{H+1:T} \in \{-1,+1\}^{p(T-H)}.
\end{align*}
Then, we decompose the regret as follows:
\begin{align}
    R_T(\pi) = \frac{1}{2} \left( R_{1,T} + R_{2,T} + R_{3,T}\right)
\end{align}
where
\begin{align*}
    R_{1,T} &:= (\vu_{0:T}^\star)^\top \vS_T \vu_{0:T}^\star - (\vutil_{H+1:T})^\top \vS_{T-H-1} \vutil_{H+1:T} \\
    R_{2,T} &:= (\vutil_{H+1:T})^\top \vS_{T-H-1} \vutil_{H+1:T} - (\vu_{H+1:T}^\pi)^\top \vShat_{T-H-1} \vu_{H+1:T}^\pi \\
	R_{3,T} &:= (\vu_{H+1:T}^\pi)^\top (\vShat_{T-H-1}-\vS_{T-H-1}) \vu_{H+1:T}^\pi.
\end{align*}

\subsection{Proof of Proposition 5}
\begin{proposition}
    Let $\rho\in(\rho(\vA),1)$ be given. Then
    \begin{align*}
        R_{1,T} \leq 2p\kappa^2 \left( \alpha H + \beta \right) ,
    \end{align*}
    for $\alpha = 1+\frac{\phi(\vA,\rho)\rho}{1-\rho}$, $\beta = \frac{\phi(\vA,\rho)\rho}{(1-\rho)^2} +1$, and $\kappa =\max\{ \|\vB\|, \|\vC\|\}$.
\end{proposition}

\begin{proof}
    Observe that $\vS_{T-H-1}$ is a submatrix of $\vS_T$.
    Define $\vStil_T$ by zero-padding the matrix $\vS_{T-H-1}$ so that its size is the same as $\vS_T$, 
	that is, 
    \begin{align*}
	    	\vStil_T &= 
        \begin{bmatrix}
            \vS_{T-H-1} & \vzero_{p(T-H) \times p(H+1)} \\
            \vzero_{p(H+1)\times p(T-H)} & \vzero_{p(H+1)\times p(H+1)}
        \end{bmatrix}.
    \end{align*}
    Let $\vDelta_T=\vS_T - \vStil_T$, then
    \begin{align*}
        R_{1,T} &= (\vu_{0:T}^\star)^\top \vS_T \vu_{0:T}^\star - (\vutil_{H+1:T})^\top \vS_{T-H-1} \vutil_{H+1:T} \\
        &= \max_{\substack{\|\vu\|_\infty=1 \\ \vu \in \R^{p(T+1)}}} \vu^\top \vS_T \vu - \max_{\substack{\|\vv\|_\infty=1 \\ \vv \in \R^{p(T-H)}}} \vv^\top \vS_{T-H-1} \vv \\
        &= \max_{\substack{\|\vu\|_\infty=1 \\ \vu \in \R^{p(T+1)}}} \vu^\top \vS_T \vu - \max_{\substack{\|\vv\|_\infty=1 \\ \vv \in \R^{p(T+1)}}} \vv^\top \vStil_{T} \vv \\
        &= \max_{\substack{\|\vu\|_\infty=1 \\ \vu \in \R^{p(T+1)}}} \min_{\substack{\|\vv\|_\infty=1 \\ \vv \in \R^{p(T+1)}}} \left[ \vu^\top (\vStil_T + \vDelta_T) \vu -  \vv^\top \vStil_{T} \vv \right]\\
        &\leq \max_{\substack{\|\vu\|_\infty=1 \\ \vu \in \R^{p(T+1)}}} \vu^\top \vDelta_T \vu \tag{take $\vv=\vu$}
    \end{align*}
    To be concrete for $\vDelta_T$, observe that
    \begin{align*}
    	\vDelta_T = 
    	\begin{bmatrix} 
    		\vzero_{p(T-H)\times p(T-H)} & \vR \\
    		\vR^\top & \vS_{\rm tail}
    	\end{bmatrix}
    \end{align*}
    where
    \begin{align*}
    	\vR = 
    	\begin{bmatrix}
    		\vC\vA^{T-H-1}\vB & \vC\vA^{T-H}\vB & \cdots &  \vC\vA^{T-1}\vB \\
    		\vC\vA^{T-H-2}\vB & \vC\vA^{T-H-1}\vB & \cdots & \vC\vA^{T-2}\vB \\
    		\vdots & \vdots & \ddots & \vdots \\
    		\vC\vB & \vC\vA\vB & \cdots & \vC\vA^H\vB
    	\end{bmatrix}
    	\in \R^{p(T-H) \times p(H+1)}
    \end{align*}
    and
    \begin{align*}
    	\vS_{\rm tail} = 
    	\begin{bmatrix}
    		\vzero_{p\times p} & \vC\vB & \vC\vA\vB & \cdots & \vC\vA^{H-2}\vB & \vC\vA^{H-1}\vB \\
    		(\vC\vB)^\top & \vzero_{p\times p} & \vC\vB & \cdots & \vC\vA^{H-3}\vB & \vC\vA^{H-2}\vB \\
    		(\vC\vA\vB)^\top & (\vC\vB)^\top & \vzero_{p\times p} & \cdots & \vC\vA^{H-4}\vB & \vC\vA^{H-3}\vB \\
    		\vdots & \vdots & \vdots & \ddots & \vdots & \vdots \\
    		(\vC\vA^{H-2}\vB)^\top & (\vC\vA^{H-3}\vB)^\top & (\vC\vA^{H-4}\vB)^\top & \cdots & \vzero_{p\times p} & \vC\vB \\
    		(\vC\vA^{H-1}\vB)^\top & (\vC\vA^{H-2}\vB)^\top & (\vC\vA^{H-3}\vB)^\top & \cdots & (\vC\vB)^\top & \vzero_{p\times p} 
    	\end{bmatrix} \in \R^{p(H+1) \times p(H+1)}
    \end{align*}
    By Proposition~\ref{prop:upperbound}, we can find an upper bound in terms of the spectral norms of the blocks as follows:
    \begin{align*}
        \max_{\substack{\|\vu\|_\infty = 1\\ \vu \in \R^{p(T+1)}}} \vu^\top \vDelta_T \vu
        &\leq 
        \underbrace{2p \sum_{i=0}^H \sum_{j=0}^{T-H-1} \|\vC\vA^{i+j}\vB\|}_{\text{blocks in } \vR,\vR^\top} + 
        \underbrace{2p \sum_{l=1}^H \sum_{k=1}^l \|\vC\vA^{k-1}\vB\|}_{\text{blocks in } \vS_{\rm tail}} 
    \end{align*}
    We use Corollary~\ref{cor:lem:finite-gelfands} in order to bound the upper bound above.
    Let $\rho \in (\rho(\vA),1)$ be given.
    Then the term related to the blocks in $\vR$ and $\vR^\top$ is bounded by, 
    \begin{align*}
        2p \sum_{i=0}^H \sum_{j=0}^{T-H-1} \|\vC\vA^{i+j}\vB\|
        &= 2p\|\vC\vB\| + 2p \sum_{j=1}^{T-H-1} \|\vC\vA^j\vB\| + 2p \sum_{i=1}^H \sum_{j=0}^{T-H-1} \|\vC\vA^{i+j}\vB\| \\
        &\leq 2p\|\vB\|\|\vC\| \left(  1+ \sum_{j=1}^{T-H-1} \|\vA^j\| + \sum_{i=1}^H \sum_{j=0}^{T-H-1} \|\vA^{i+j}\| \right) \\
        &\leq 2p\|\vB\|\|\vC\| \left( 1+\phi(\vA,\rho) \sum_{j=1}^{T-H-1} \rho^j + \phi(\vA,\rho) \sum_{i=1}^H \sum_{j=0}^{T-H-1} \rho^{i+j}\right) \\
        &\leq 2p\|\vB\|\|\vC\| \left( 1+\phi(\vA,\rho) \sum_{j=1}^\infty \rho^j + \phi(\vA,\rho) \sum_{i=1}^\infty \rho^i\sum_{j=0}^\infty \rho^{j} \right) \\
        &\leq 2p\|\vB\|\|\vC\| \left( 1+ \frac{\phi(\vA,\rho)\rho}{1-\rho} + \frac{\phi(\vA,\rho)\rho}{(1-\rho)^2} \right)
    \end{align*}
    Similarly, the term related to the blocks in $\vS_{\rm tail}$ is bounded by,
    \begin{align*}
        2p \sum_{l=1}^H \sum_{k=1}^l \|\vC\vA^{k-1}\vB\| 
        &= 2pH\|\vC\vB\| + 2p \sum_{l=1}^{H-1} \sum_{k=1}^l \|\vC\vA^k\vB\| \\
        &\leq 2p\|\vB\|\|\vC\| \left(H +\phi(\vA,\rho)\sum_{l=1}^{H-1}\sum_{k=1}^l\rho^k \right) \\
        &\leq 2p\|\vB\|\|\vC\| \left(H +\phi(\vA,\rho) \sum_{l=1}^{H-1} \sum_{k=1}^\infty \rho^k\right) \\
        &\leq 2p\|\vB\|\|\vC\| \left(H +\phi(\vA,\rho) \sum_{l=1}^{H-1} \frac{\rho}{1-\rho}\right) \\
        &= 2p\|\vB\|\|\vC\| \left(H + \frac{(H-1)\phi(\vA,\rho)\rho}{1-\rho}\right)
    \end{align*}
    Combining two bounds, we get
    \begin{align*}
        R_{1,T} \leq \max_{\substack{\|\vu\|_\infty = 1\\ \vu \in \R^{p(T+1)}}} \vu^\top \vDelta_T \vu
        &\leq
        2p\|\vB\|\|\vC\| \left( \left( 1+ \frac{\phi(\vA,\rho)}{1-\rho}\right)H + \frac{\phi(\vA,\rho)\rho}{(1-\rho)^2}+1\right)
    \end{align*}
    which completes the proof.
\end{proof}

\subsection{Proof of Proposition 6}
\begin{proposition}
    Let $\rho \in (\rho(\vA),1)$ be given and $\epsilon>0$ be the high-probability parameter estimation error, 
    i.e., $\|\vG-\vGhat\|_F \le \epsilon$. Then, with high probability, 
    \begin{align*}
        \max\{R_{2,T}, R_{3,T} \} \leq 2p(T-H) \left( \epsilon + \kappa^2 \gamma_L\right)
    \end{align*}
    where $\kappa = \max\{\|\vB\|, \|\vC\|\}$ and $\gamma_L = \frac{\phi(\vA,\rho)\rho^L}{1-\rho}$.
\end{proposition}

\begin{proof}
    Observe the following bounds for $R_{2,T}$ and $R_{3,T}$.
    \begin{align*}
        R_{2,T} 
        &= (\vutil_{H+1:T})^\top \vS_{T-H-1} \vutil_{H+1:T} - (\vu_{H+1:T}^\pi)^\top \vShat_{T-H-1} \vu_{H+1:T}^\pi \\
        &= \max_{\|\vu\|_\infty=1} \vu^\top \vS_{T-H-1} \vu - \max_{\|\vv\|_\infty=1} \vv^\top \vShat_{T-H-1} \vv \\
        &= \max_{\|\vu\|_\infty=1} \min_{\|\vv\|_\infty=1} \left[\vu^\top \vS_{T-H-1} \vu - \vv^\top \vShat_{T-H-1} \vv \right]\\
        &\leq \max_{\|\vu\|_\infty=1} \vu^\top ( \vS_{T-H-1} - \vShat_{T-H-1} )\vu \tag{take $\vv = \vu$} \\
        & \\
        R_{3,T} &= (\vu_{H+1:T}^\pi)^\top (\vShat_{T-H-1}-\vS_{T-H-1}) \vu_{H+1:T}^\pi \\
    	&\leq \max_{\|\vu\|_\infty = 1} \vu^\top (\vShat_{T-H-1} - \vS_{T-H-1}) \vu
    \end{align*}
    According to Proposition~\ref{prop:upperbound}, we now see that $R_{2,T}$ and $R_{3,T}$ can be bounded by a same quantity, which will be specified in this proof.

    Let $\vMhat_{T-H-1}$ be the upper diagonal part of $\vShat_{T-H-1}$, that is, $\vShat_{T-H-1} = \vMhat_{T-H-1} + \vMhat_{T-H-1}^\top$.
    Recall that $\vS_{T-H-1} = \vM_{T-H-1} + \vM^\top_{T-H-1}$ where 
    \begin{align*}
        \vM_{T-H-1} = 
        \begin{bmatrix}
            \vzero_{p\times p} & \vC\vB & \vC\vA\vB &  \cdots & \cdots & \vC\vA^{T-H-2} \vB \\
            \vzero_{p\times p} & \vzero_{p\times p} & \vC\vB & \cdots & \cdots & \vC\vA^{T-H-3} \vB \\
            \vdots & \ddots & \ddots & \ddots &  & \vdots \\
            \vdots &  & \ddots & \vzero_{p\times p} & \vC\vB & \vC\vA\vB \\
            \vdots &  &  & \ddots & \vzero_{p\times p} & \vC\vB  \\
            \vzero_{p\times p} & \vzero_{p\times p} & \vzero_{p\times p} & \cdots & \vzero_{p\times p} & \vzero_{p\times p}
        \end{bmatrix} \in \R^{p(T-H) \times p(T-H)}.
    \end{align*}
    Observe that 
    \begin{align*}
        \vu^\top (\vS_{T-H-1} - \vShat_{T-H-1}) \vu 
    	= 2\vu^\top (\vM_{T-H-1} - \vMhat_{T-H-1}) \vu.
    \end{align*}
    Let $\vDelta_{ij}$ denote the $(i,j)$-th $p{\times}p$ block of $\vM_{T{-}H{-}1}{-}\vMhat_{T{-}H{-}1}$.
    By Proposition~\ref{prop:upperbound}, 
    \begin{align*}
        \max_{\|\vu\|_\infty = 1} \vu^\top (\vM_{T-H-1} - \vMhat_{T-H-1}) \vu 
        &\leq p \left( \sum_{(i,j) \in \Ical} \|\vDelta_{ij}\| + \sum_{(i,j) \in \Ical^c} \|\vDelta_{ij}\|_F\right)
    \end{align*}
    where $\Ical = \{(i,j) \colon L+1 \le i+L<j \leq T-H\}$.
    Define $\Ical_k = \{(k,j) \colon (k,j) \in \Ical\}$ and $\Jcal_k = \{(k,j) \colon (k,j) \in \Ical^c\}$ for $k=1,\dots,T-H$, then 
    \begin{align*}
        \Ical = \dot{\cup}_{k=1}^{T{-}H{-}L{-}1} \Ical_k, \quad \Jcal = \dot{\cup}_{k=1}^{T-H} \Jcal_k,
    \end{align*}
    and $\Ical_k = \varnothing$ for $k=T-H-L,\dots,T-H$.
    For each $k=1,\dots,T{-}H{-}L{-}1$, 
    \begin{align*}
        \sum_{(k,j) \in \Ical_k} \|\vDelta_{ij}\| + \sum_{(k,j) \in \Jcal_k} \|\vDelta_{ij}\|_F 
        &= \sum_{\ell=L+1}^{T-H-k}\|\vC\vA^{\ell-1}\vB\| + \|\vG-\vGhat\|_F
    \end{align*}
    and for $k=T-H-L, \dots, T-H$,
    \begin{align*}
        \sum_{(k,j) \in \Ical_k} \|\vDelta_{ij}\| + \sum_{(k,j) \in \Jcal_k} \|\vDelta_{ij}\|_F 
        = 0 + \sum_{(k,j) \in \Jcal_k} \|\vDelta_{ij}\|_F 
        \leq \|\vG-\vGhat\|_F
    \end{align*}
    Therefore, we obtain
    \begin{align*}
        \max_{\|\vu\|_\infty =1 }\vu^\top (\vS_{T-H-1} - \vShat_{T-H-1}) \vu 
    	&= 2 \max_{\|\vu\|_\infty = 1}\vu^\top (\vM_{T-H-1} - \vMhat_{T-H-1}) \vu \\
        &\leq 2p\left( \sum_{(i,j) \in \Ical} \|\vDelta_{ij}\| + \sum_{(i,j) \in \Ical^c} \|\vDelta_{ij}\|_F\right) \\
        &= 2p \sum_{k=1}^{T-H}\left( \sum_{(k,j) \in \Ical_k} \|\vDelta_{ij}\| + \sum_{(k,j) \in \Jcal_k} \|\vDelta_{ij}\|_F\right) \\
        &\leq 2p\sum_{k=1}^{T{-}H{-}L{-}1} \left( \sum_{\ell=L+1}^{T-H-k}\|\vC\vA^{\ell-1}\vB\| + \|\vG-\vGhat\|_F\right) + 2p(L+1)\|\vG-\vGhat\|_F \\
        &\leq 2p\sum_{k=1}^{T{-}H{-}L{-}1} \sum_{\ell=L+1}^{T-H-k}\|\vC\vA^{\ell-1}\vB\| + 2p(T-H)\|\vG-\vGhat\|_F
    \end{align*}
    We again use Corollary~\ref{cor:lem:finite-gelfands}. 
    Given $\rho \in (\rho(\vA), 1)$, 
    \begin{align*}
        \sum_{k=1}^{T{-}H{-}L{-}1} \sum_{\ell=L+1}^{T-H-k} \|\vC\vA^{\ell-1}\vB\| 
        &\leq \phi(\vA,\rho)\|\vB\|\|\vC\|\sum_{k=1}^{T{-}H{-}L{-}1} \sum_{\ell=L+1}^{T-H-k}  \rho^{\ell-1} \\
        &= \phi(\vA,\rho)\|\vB\|\|\vC\|\sum_{k=1}^{T{-}H{-}L{-}1} \frac{\rho^L - \rho^{T-H-k}}{1-\rho}\\
        &= \phi(\vA,\rho)\|\vB\|\|\vC\| \left(\frac{(T-H-L-1)\rho^L}{1-\rho} - \sum_{k=1}^{T-H-L-1}\frac{\rho^{T-H-k}}{1-\rho}\right) \\
        &\leq (T-H)\phi(\vA,\rho)\|\vB\|\|\vC\| \frac{\rho^L}{1-\rho}
    \end{align*}
    In conclusion, we get a final upper bound as follows:
    \begin{align*}
        \max_{\|\vu\|_\infty =1 }\vu^\top (\vS_{T-H-1} - \vShat_{T-H-1}) \vu 
        &\leq 2p(T-H) \left( \|\vB\|\vC\| \frac{\phi(\vA,\rho)\rho^L}{1-\rho} + \|\vG-\vGhat\|_F\right)
    \end{align*}
    This completes the proof.
\end{proof}

\subsection{Miscellaneous}

\begin{lemma} \label{lem:frob}
	For a square $n\times n$ matrix $\vM=(M_{ij})$, we have
	\begin{align*}
		\vx^\top \vM \vx \leq n \|\vM\|_F 
		\quad 
		\text{and}
		\quad
		\vx^\top \vM\vx \leq n \|\vM\|_2
	\end{align*}
	whenever $\|\vx\|_\infty \leq 1$.
\end{lemma}
\begin{proof}
Assuming $\|\vx\|_\infty \leq 1$, the first result can be proved as follows:
\begin{align*}
	\vx^\top \vM \vx &= \sum_{i,j=1}^n M_{ij} x_i x_j 
	\leq \sum_{i,j=1}^n |M_{ij}| 
	\leq \sqrt{n^2} \sqrt{\sum_{i,j=1}^n |M_{ij}|^2} 
	= n \|\vM\|_F
\end{align*}
where the second inequality is obtained using Cauchy-Schwarz inequality.
The second result can be readily obtained:
\begin{align*}
	\vx^\top \vM \vx \leq \|\vM\|_2 \|\vx\|_2^2 \leq n \|\vM\|_2
\end{align*}
because $\|\vx\|_\infty=1$ implies $\|\vx\|_2 \leq \sqrt{n}$.
\end{proof}

\begin{proposition} \label{prop:upperbound}
	Let $\vA$ be a block matrix given by
	\begin{align*}
		\vA = 
		\begin{bmatrix}
			\vA_{11} & \vA_{12} & \cdots & \vA_{1n} \\
			\vA_{21} & \vA_{22} & \cdots & \vA_{2n} \\
			\vdots & \vdots & \ddots & \vdots \\
			\vA_{n1} & \vA_{n2} & \cdots & \vA_{nn}
		\end{bmatrix}
	\end{align*}
	where $\vA_{ij} \in \R^{p\times p}$.
	Let $\Ical$ be any subset of $\{(i,j) \colon 1 \leq i,j \leq n\}$.
	Then we have the following:
	\begin{align}
		\max_{\|\vx\|_\infty=1} \vx^\top \vA \vx 
		&\leq p \left(\sum_{(i,j) \in \Ical} \|\vA_{ij}\|_2 + \sum_{(i,j) \in \Ical^c} \|\vA_{ij}\|_F \right)
	\end{align}
	where $\vx \in \R^{pn}$.
\end{proposition}
\begin{proof}
Let $\vx = \begin{bmatrix} \vx_1^\top & \cdots & \vx_n^\top \end{bmatrix}^\top$
where $\vx_i \in \R^p$ for each $i$.
First, observe the equality:
\begin{align*}
\vx^\top \vA \vx = \sum_{i,j=1}^n \vx_i^\top \vA_{ij} \vx_j.
\end{align*}
For arbitrarily chosen index pair $(i,j)$, Lemma~\ref{lem:frob} implies that 
\begin{align*}
	\vx_i^\top \vA_{ij} \vx_j \leq p \|\vA_{ij}\|_F 
	\quad
	\text{and}
	\quad
	\vx_i^\top \vA_{ij} \vx_j \leq p \|\vA_{ij}\|_2.
\end{align*}
This completes the proof.
\end{proof}

\section{Numerical Experiments}

In this section, we give further details regarding numerical experiments. In Section~\ref{supp:subsec:GW}, we give a lower bound of the expected output of Goemans-Williamson random hyperplane rounding method, which is similar but different from MaxCut problems. 
In Section~\ref{supp:subsec:sys-id}, we describe further discussion about the experimental results in parameter estimation.
In Section~\ref{supp:subsec:add-exp}, we design additional numerical experiments to show that our practical approach works for dense matrices $\vA,\vB,\vC$, and we discuss the experimental results with different spectral radii of $\vA$. The code for experiments can be found in \url{https://github.com/sdean-group/EtC-latent-bandits}.

\subsection{Goemans-Williamson Random Hyperplane Rounding} \label{supp:subsec:GW}

We recall the problem stated in Section~\ref{subsec:SDP-GW}.
\begin{equation} \tag*{(19)}
    \begin{split} 
        \text{maximize }& \tr(\vW\vX) \\
        \text{subject to }& \vX \succeq 0, \, {\rm rank}(\vX) =1 \\
        & \vX_{ii} = 1, \quad i=1,\dots,n. 
    \end{split}
\end{equation}
Its semidefinite relaxation is 
\begin{equation} \label{supp:eqn:SDP-relax}
    \begin{split} 
        \text{maximize }& \tr(\vW\vX) \\
        \text{subject to }& \vX \succeq 0 \\
        & \vX_{ii} = 1, \quad i=1,\dots,n. 
    \end{split}
\end{equation}
    
In Section~\ref{subsec:SDP-GW}, we discussed Goemans-Williamson random hyperplane rounding method to find a feasible point of \eqref{eqn:simple-form-matrix} from its relaxed problem~\eqref{supp:eqn:SDP-relax}. 
For MaxCut problems, \texttt{SDP+GW} achieves an $\alpha$-approximation algorithm with $\alpha \approx 0.87856$. 
This means that, for any given instance of MaxCut problem, the algorithm is guaranteed on average to return a solution whose value is at least 87.856\% of the true optimum.
In the next proposition, we derive a similar result where $\alpha$ depends on the matrix $\vW$ in \eqref{eqn:simple-form-matrix}.

Let $\nu$ and $\nu_{\rm rlx}$ be the optimal value of the original problem~\eqref{eqn:simple-form-matrix} and its relaxed problem~\eqref{supp:eqn:SDP-relax}, respectively. 
It is clear that $\nu_{\rm rlx} \ge \nu$. 
Let $\vX^\star = \vV^\top \vV$ be a solution to \eqref{supp:eqn:SDP-relax} where we denote $\vv_i$ to be the $i$th column of $\vV$.
Note that the constraint $\vX_{ii}=1$ implies that $\vv_i^\top \vv_i=1$.
According to Goemans-Williamson random hyperplane rounding method, for $\vr \sim \Ncal(0,\vI_n)$, the point $\vx_f$ where
\begin{align*}
    \vx_f = 
    \begin{bmatrix}
        \text{sign}(\vv_1^\top \vr) & \cdots & \text{sign}(\vv_n^\top \vr)
    \end{bmatrix}^\top
\end{align*}
is a feasible point of \eqref{eqn:simple-form-matrix}.
The expected value $\nu_{\rm GW}$ from the rounding algorithm is
\begin{align*}
    \nu_{\rm GW} = \E\left[\tr(\vW\vx_f\vx_f^\top)\right] 
    &= \tr(\vW \vX_f)
\end{align*}
where $(\vX_f)_{ij} = \E\left[\text{sign}(\vv_i^\top \vr)\text{sign}(\vv_j^\top \vr)\right]$.

\begin{proposition}
    The expected value $\nu_{\rm GW}$ has a lower bound that depends on $\vW$.
    Specifically, we have
    $$\nu_{\rm GW} \ge \alpha \nu_{\rm rlx} -(1-\alpha) \sum_{i=1}^n \sum_{j=1}^n |\vW_{ij}|$$ where 
    $$\alpha = \min_{x\in(-1,1)}\frac{1-\frac{\arccos\left(x\right)}{\pi}}{\frac{1+x}{2}}= \min_{x \in(-1,1)} \frac{\frac{\arccos\left(x\right)}{\pi}}{\frac{1-x}{2}}\approx 0.87856.$$
\end{proposition}

\begin{proof} 
    Observe that we have
    \begin{align*}
        (\vX_f)_{ij} = \E\left[\text{sign}(\vv_i^\top \vr)\text{sign}(\vv_j^\top \vr)\right]
        &= \Pr\left(\text{sign}(\vv_i^\top \vr) = \text{sign}(\vv_j^\top \vr)\right) + (-1) \Pr\left(\text{sign}(\vv_i^\top \vr) \neq \text{sign}(\vv_j^\top \vr)\right) \\
        &= 1 - 2\Pr\left(\text{sign}(\vv_i^\top \vr) \neq \text{sign}(\vv_j^\top \vr)\right) \\
        &= 1-\frac{2}{\pi} \arccos\left(\vv_i^\top \vv_j\right)
    \end{align*}
    where $\theta_{ij} = \arccos\left(\vv_i^\top \vv_j \right)$.
    The last equality holds since each $\vv_i$ is a unit vector.
    Therefore, we have 
    \begin{align*}
        \nu_{\rm GW} = \tr(\vW\vX_f) = \sum_{i=1}^n \sum_{j=1}^n \vW_{ij} \left( 1-\frac{2}{\pi}\arccos\left(\vv_i^\top \vv_j\right) \right).
    \end{align*}
    On the other hand, the optimal value to \eqref{supp:eqn:SDP-relax} is 
    \begin{align} \label{supp:eqn:SDP-relax-sol}
        \nu_{\rm rlx} = \tr(\vW\vX^\star) = \tr(\vW\vV^\top \vV) = \tr(\vV\vW\vV^\top) = \tr \left(\sum_{i=1}^n \sum_{j=1}^n \vW_{ij} \vv_i\vv_j^\top\right) = \sum_{i=1}^n \sum_{j=1}^n \vW_{ij} \vv_i^\top \vv_j
    \end{align}
    Observe the two inequalities that is valid for all $t \in (-1,1)$:
    \begin{align*}
        1- \frac{\arccos(t)}{\pi} \ge \alpha \frac{1+t}{2}, \quad
        \frac{\arccos(t)}{\pi} \ge \alpha \frac{1-t}{2}
    \end{align*}
    where $\alpha = 0.87856...$.
    From each inequality, we can deduce the two inequalities:
    \begin{align*}
        1-\frac{2}{\pi}\arccos(t) \quad &\geq \quad  \alpha(1+t) -1 = \alpha t + (\alpha-1) \\
        \frac{2}{\pi}\arccos(t)-1 \quad &\ge \quad \alpha(1-t) - 1 = -\alpha t +(\alpha-1).
    \end{align*}
    Combining two inequalities together, we have
    \begin{align*}
        s\left(1-\frac{2}{\pi}\arccos(t)\right) \ge s\alpha t + (\alpha - 1)
    \end{align*}
    for $s = \pm 1$.
    Therefore, letting $s_{ij} = \text{sign}(\vW_{ij})$, we have
    \begin{align*}
        \nu_{\rm GW} &= \sum_{i=1}^n \sum_{j=1}^n |\vW_{ij}|s_{ij} \left( 1- \frac{2}{\pi}\arccos(\vv_i^\top \vv_j)\right) \\
        &\ge \sum_{i=1}^n \sum_{j=1}^n \left(|\vW_{ij}|s_{ij}\alpha (\vv_i^\top \vv_j) + (\alpha -1) \right) \\
        &= \alpha \nu_{\rm rlx} -(1-\alpha) \sum_{i=1}^n \sum_{j=1}^n |\vW_{ij}|
    \end{align*}
    by \eqref{supp:eqn:SDP-relax-sol}. This completes the proof.
\end{proof}

\subsection{Parameter Estimation} \label{supp:subsec:sys-id}

In Section~\ref{subsec:param-est-plots}, we discussed how the relative estimation error changes with different truncation lengths $L$ under systems with spectral radii $\rho(\vA)=0.1$ and $\rho(\vA)=0.9$. 
In this section, we discuss the experimental result in detail.

In the experiment, we consider the case when $n=5$ and $p=3$. The Schur-stable matrix $\vA$ is generated with i.i.d. $\Ncal(0,1/n)$ entries and scaled to a desired spectral radius.
The matrices $\vB$ and $\vC$ are generated with i.i.d. $\Ncal(0,1/n)$ and $\Ncal(0,1/p)$, respectively. The noise processes are chosen to be Gaussian, that is, $\vw_t \distas \Ncal\big(\vzero_n, (0.05)^2 \vI_n \big)$ and $z_t \distas \Ncal\big(0, (0.05)^2 \big)$.
We repeated each experiment 20 times with different noise seeds.

In Figures~\ref{fig:param-est-rho-p1} and \ref{fig:param-est-rho-p9}, we can see that the relative error of Markov parameters tends to decrease as the exploration length $H$ increases. 
It is a natural phenomenon which explains that more data samples give better estimation. 
When $\rho(\vA)=0.1$ (consider Figure~\ref{fig:param-est-rho-p1}), we can also observe that larger truncation length $L$ gives larger error. 
If we have larger $L$, we have more parameters (i.e., $\vC\vB, \vC\vA\vB, \dots, \vC\vA^{L-1}\vB$) to estimate, but the spectral radius $\rho(\vA)$ too small to capture longer memory of the system, giving larger estimation error.
On the other hand, when $\rho(\vA)=0.9$ (consider Figure~\ref{fig:param-est-rho-p9}), we can see that it tends to have smaller errors when the truncation length $L$ is larger. It is because the spectral radius $\rho(\vA)$ is large enough to capture memory of the system and give help to estimate more number of Markov parameters.

Moreover, we can observe double-descent phenomenon~\cite{nakkiran2020optimal} in both plots, which happens when a model whose number of parameters is about the same as the number of data used to estimate the model. 
In our problem, the regression model has $p^2L$ number of parameters, and the number of samples used is $H-L$, implying that the peak would be shown at $H=L+p^2L$.
We can see the peaks at $H=300$ and $H=500$ for $L=30$ and $L=50$, respectively, in Figures~\ref{fig:param-est-rho-p1} and \ref{fig:param-est-rho-p9}.
For $L=6,12,18$, the peak will be obtained at $H=60, 120, 180$, but it is not shown in the plot because we only plotted the curves at $H\in \{100,150,\dots,2000\}$.

\subsection{Additional Experiments} \label{supp:subsec:add-exp}

\begin{figure*}[ht]
    \centering
    \begin{subfigure}[t]{0.19\textwidth}
        \centering
        \includegraphics[width=\linewidth]{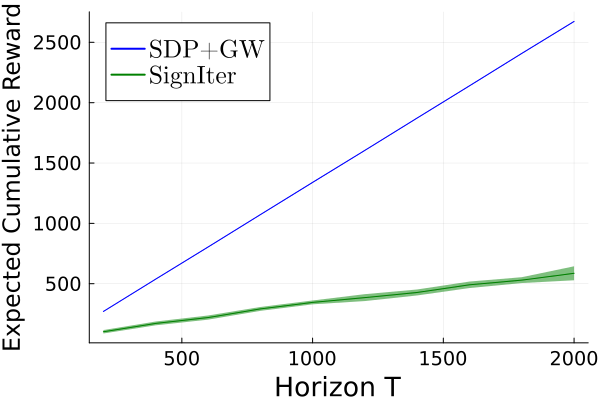}
        \caption{$\rho(\vA)=0.1$}
        \label{supp:fig:regret-benchmark-rho1}
    \end{subfigure}
    \hfill
    \begin{subfigure}[t]{0.19\textwidth}
        \centering
        \includegraphics[width=\linewidth]{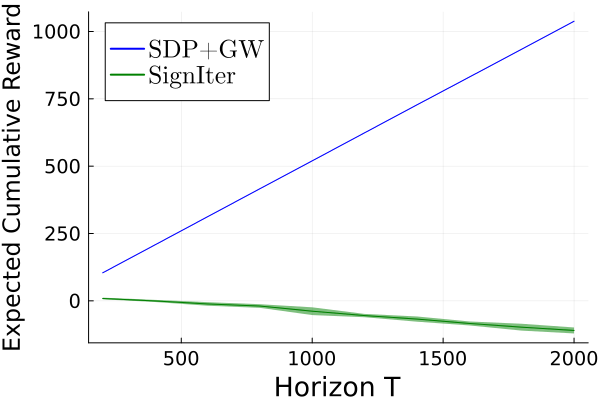}
        \caption{$\rho(\vA)=0.3$}
        \label{supp:fig:regret-benchmark-rho3}
    \end{subfigure}
    \hfill\begin{subfigure}[t]{0.19\textwidth}
        \centering
        \includegraphics[width=\linewidth]{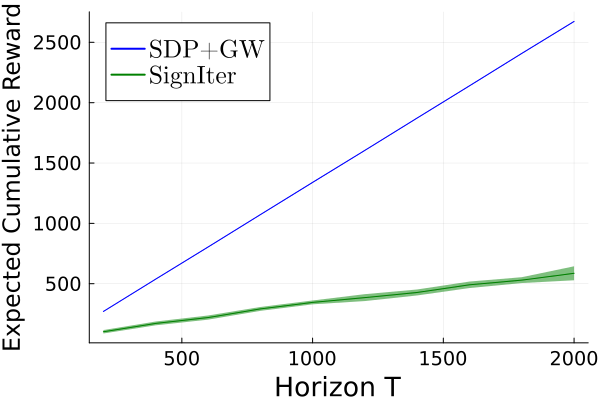}
        \caption{$\rho(\vA)=0.5$}
        \label{supp:fig:regret-benchmark-rho5}
    \end{subfigure}
    \hfill\begin{subfigure}[t]{0.19\textwidth}
        \centering
        \includegraphics[width=\linewidth]{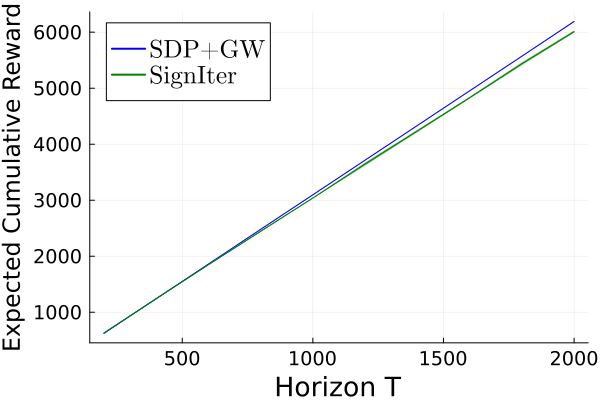}
        \caption{$\rho(\vA)=0.7$}
        \label{supp:fig:regret-benchmark-rho7}
    \end{subfigure}
    \hfill\begin{subfigure}[t]{0.19\textwidth}
        \centering
        \includegraphics[width=\linewidth]{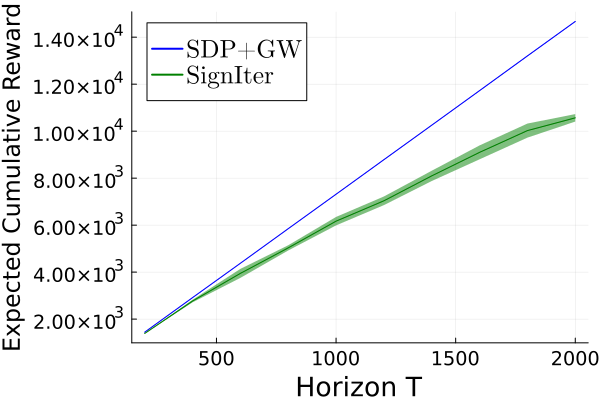}
        \caption{$\rho(\vA)=0.9$}
        \label{supp:fig:regret-benchmark-rho9}
    \end{subfigure}
    \caption{Expected cumulative reward under the oracle benchmark, approximated by semidefinite relaxation and Goemans-Williamson rounding (\texttt{SDP+GW}) and by the sign-iteration method (\texttt{SignIter}), in different spectral radii $\rho(\vA)$.}
    \label{supp:fig:regret-benchmark}
\end{figure*}

\begin{figure*}[ht]
    \centering
    \begin{subfigure}[t]{0.19\textwidth}
        \centering
        \includegraphics[width=\linewidth]{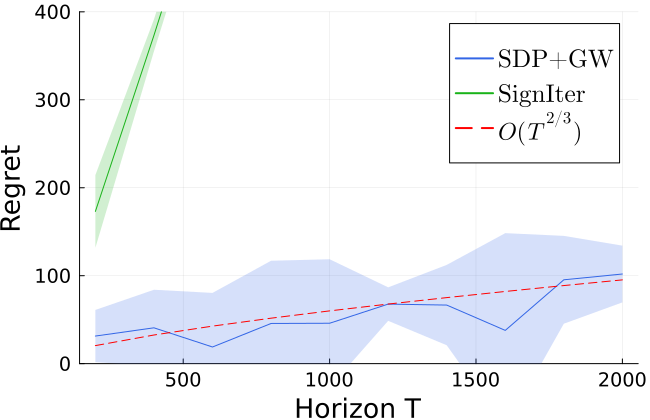}
        \caption{$\rho(\vA)=0.1$}
        \label{supp:fig:regret-rho1}
    \end{subfigure}
    \hfill
    \begin{subfigure}[t]{0.19\textwidth}
        \centering
        \includegraphics[width=\linewidth]{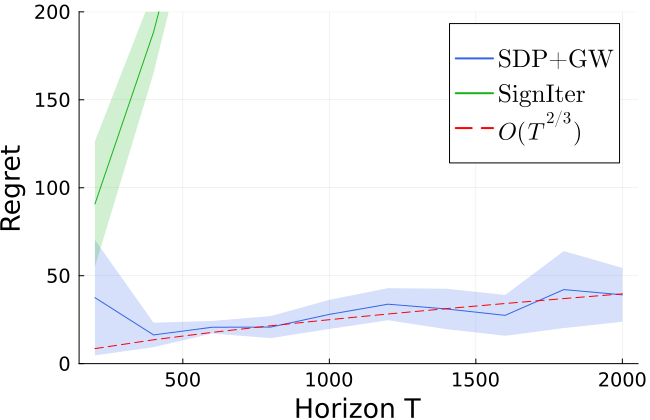}
        \caption{$\rho(\vA)=0.3$}
        \label{supp:fig:regret-rho3}
    \end{subfigure}
    \hfill
    \begin{subfigure}[t]{0.19\textwidth}
        \centering
        \includegraphics[width=\linewidth]{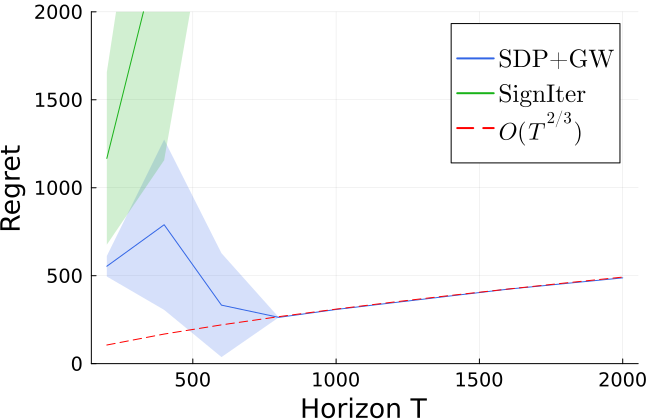}
        \caption{$\rho(\vA)=0.5$}
        \label{supp:fig:regret-rho5}
    \end{subfigure}
    \hfill
    \begin{subfigure}[t]{0.19\textwidth}
        \centering
        \includegraphics[width=\linewidth]{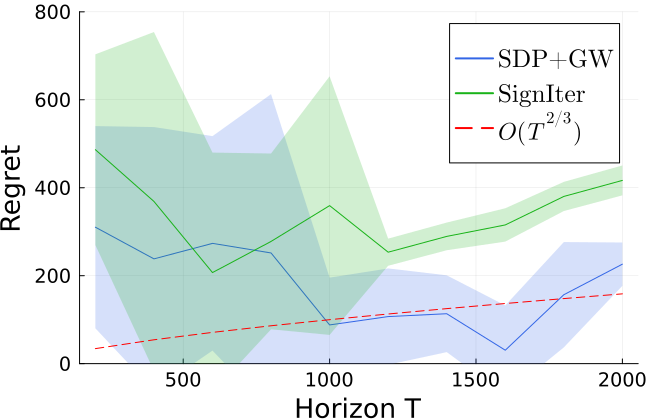}
        \caption{$\rho(\vA)=0.7$}
        \label{supp:fig:regret-rho7}
    \end{subfigure}
    \hfill
    \begin{subfigure}[t]{0.19\textwidth}
        \centering
        \includegraphics[width=\linewidth]{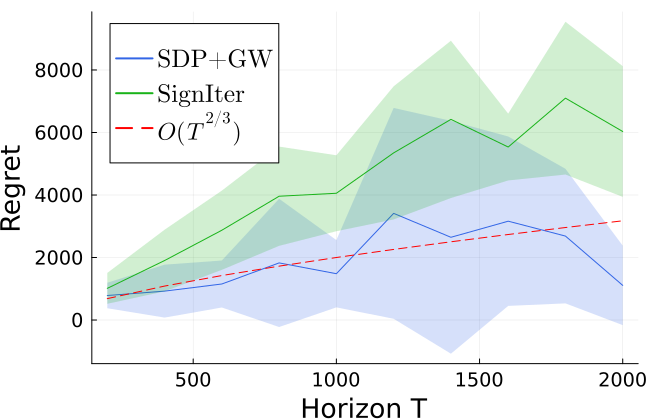}
        \caption{$\rho(\vA)=0.9$}
        \label{supp:fig:regret-rho9}
    \end{subfigure}
    \caption{The regret curves of the explore-then-commit algorithm measured from \texttt{SDP+GW} and \texttt{SignIter} against the \texttt{SDP+GW} oracle benchmark, compared with the theoretical $\tilde{\Ocal}(T^{2/3})$ rate.}
    \label{supp:fig:regret-plots}
\end{figure*}

In Section~\ref{subsec:exp-regret}, we used a simple latent dynamics and reward function. Especially, we set the matrix $\vA$ diagonal with small spectral radius (which is $\rho(\vA)=0.3$). 
This section is devoted to discuss whether our practical approaches, \texttt{SDP+GW} and \texttt{SignIter}, would work for dense matrices. 

In order to generate dense matrices, $\vA$, $\vB$, and $\vC$ are generated in the same way as we did in the experiment for parameter estimation described in Section~\ref{supp:subsec:sys-id} except that we use $n=3$ and $p=2$ in this experiment.
We set the noise processes to be Gaussian, that is, $\vw_t \distas \Ncal(\vzero_n, (0.01)^2\vI_n)$ and $z_t \distas \Ncal(0, (0,01)^2)$. 
We conducted our experiments for $\rho(\vA) \in \{0.1, 0.3, 0.5, 0.7, 0.9\}$ by scaling the same matrix $\vA$,
and we performed a grid search to determine constants $c_1, c_2$ in $H=c_1T^{2/3}$ and $L=c_2\log T$, calibrated at $T=1500$. The grid was chosen to be $c_1 \in \big\{T_0^\epsilon \colon T_0=1500, \; \epsilon \in \{0, \pm 0.02, \pm 0.05, \pm 0.1\}\big\}$ and $c_2 \in \{0.75, 1.0, 1.25\}$. For all cases of $\rho(\vA)$, an optimal value of $c_1$ is $1500^{-0.1} \approx 0.48$. 
On the other hand, optimal values of $c_2$ for $\rho(\vA)=0.1, 0.3, 0.5, 0.7$, and $0.9$ are 0.75, 0.75, 1.0, 1.25, and 1.25, respectively, which is a consistent result with the experiment in parameter estimation.
Each experiment was repeated 10 times with different noise seeds, except for the grid search.

Recall that a feasible solution to our optimization problems~\eqref{eqn:baseline-regret-form} and \eqref{eqn:commit-obj} is obtained by random rounding methods discussed in Section~\ref{subsec:SDP-GW} and Section~\ref{subsec:sign-iter}. 
Since the matrices are more complicated than the ones in Section~\ref{subsec:exp-regret}, we might expect that rounding methods need more trials to get feasible points that are close to optimal. Hence we increased the number of rounding trials for both Goemans-Williamson random hyperplane rounding method and sign-iteration method as $\rho(\vA)$ increases. Moreover, we also increased the maximum number of iterations for sign-iteration method.
The number of rounding trials is chosen to be 256, 320, 384, 448, and 512 for $\rho(\vA)=0.1, 0.3, 0.5, 0.7$, and $0.9$, respectively, and the maximum number of iterations for the sign-iteration method is chosen to be 200, 250, 300, 350, and 400, respectively. 

In Figure~\ref{supp:fig:regret-benchmark}, we can see that \texttt{SDP+GW} always gives the larger value of the expected cumulative reward. Therefore, we choose \texttt{SDP+GW} as our reward benchmark.
In Figure~\ref{supp:fig:regret-plots}, the regret curves are shown which are computed from two practical approaches against the \texttt{SDP+GW} oracle benchmark with different spectral radii. 
We can see that it tends to accrue more rewards when $\rho(\vA)$ is closer to $1$ because the states decay slower. 
In addition, the numerical experiment verifies that, even though the regret curve is more noisy than the simple case in Section~\ref{subsec:exp-regret}, the curves are close to our theoretical sublinear $\tilde{\Ocal}(T^{2/3})$ rates.

%% file: main.bbl
\begin{thebibliography}{}

\bibitem[Abbasi-Yadkori et~al., 2011]{abbasi2011improved}
Abbasi-Yadkori, Y., P{\'a}l, D., and Szepesv{\'a}ri, C. (2011).
\newblock Improved algorithms for linear stochastic bandits.
\newblock {\em Advances in neural information processing systems}, 24.

\bibitem[ApS, 2025]{mosek}
ApS, M. (2025).
\newblock {\em MOSEK Optimizer API for Julia 11.0.29}.

\bibitem[Bar-Shalom and Tse, 2003]{bar2003dual}
Bar-Shalom, Y. and Tse, E. (2003).
\newblock Dual effect, certainty equivalence, and separation in stochastic control.
\newblock {\em IEEE Transactions on Automatic Control}, 19(5):494--500.

\bibitem[Basu et~al., 2019]{basu2019blocking}
Basu, S., Sen, R., Sanghavi, S., and Shakkottai, S. (2019).
\newblock Blocking bandits.
\newblock {\em Advances in Neural Information Processing Systems}, 32.

\bibitem[Cella and Cesa-Bianchi, 2020]{cella2020stochastic}
Cella, L. and Cesa-Bianchi, N. (2020).
\newblock Stochastic bandits with delay-dependent payoffs.
\newblock In {\em International Conference on Artificial Intelligence and Statistics}, pages 1168--1177. PMLR.

\bibitem[Clerici et~al., 2024]{clerici2024linear}
Clerici, G., Laforgue, P., Cesa~Bianchi, N., et~al. (2024).
\newblock Linear bandits with memory.
\newblock {\em Transactions on Machine Learning Research}, (5):1--26.

\bibitem[Dean et~al., 2018]{dean2018regret}
Dean, S., Mania, H., Matni, N., Recht, B., and Tu, S. (2018).
\newblock Regret bounds for robust adaptive control of the linear quadratic regulator.
\newblock In {\em Advances in Neural Information Processing Systems}, pages 4188--4197.

\bibitem[Goemans and Williamson, 1995]{goemans1995improved}
Goemans, M.~X. and Williamson, D.~P. (1995).
\newblock Improved approximation algorithms for maximum cut and satisfiability problems using semidefinite programming.
\newblock {\em Journal of the ACM (JACM)}, 42(6):1115--1145.

\bibitem[Heidari et~al., 2016]{heidari2016tight}
Heidari, H., Kearns, M.~J., and Roth, A. (2016).
\newblock Tight policy regret bounds for improving and decaying bandits.
\newblock In {\em IJCAI}, pages 1562--1570.

\bibitem[Jun et~al., 2019]{jun2019bilinear}
Jun, K.-S., Willett, R., Wright, S., and Nowak, R. (2019).
\newblock Bilinear bandits with low-rank structure.
\newblock In {\em International Conference on Machine Learning}, pages 3163--3172. PMLR.

\bibitem[Khosravi et~al., 2023]{khosravi2023bandits}
Khosravi, K., Leme, R.~P., Podimata, C., and Tsorvantzis, A. (2023).
\newblock Bandits with deterministically evolving states.
\newblock {\em arXiv preprint arXiv:2307.11655}.

\bibitem[Kleinberg and Immorlica, 2018]{kleinberg2018recharging}
Kleinberg, R. and Immorlica, N. (2018).
\newblock Recharging bandits.
\newblock In {\em 2018 IEEE 59th Annual Symposium on Foundations of Computer Science (FOCS)}, pages 309--319. IEEE.

\bibitem[Kumar and Varaiya, 2015]{kumar2015stochastic}
Kumar, P.~R. and Varaiya, P. (2015).
\newblock {\em Stochastic systems: Estimation, identification, and adaptive control}.
\newblock SIAM.

\bibitem[Kumar et~al., 2024]{kumar2024online}
Kumar, R., Dean, S., and Kleinberg, R. (2024).
\newblock Online convex optimization with unbounded memory.
\newblock {\em Advances in Neural Information Processing Systems}, 36.

\bibitem[Lale et~al., 2020a]{lale2020logarithmic}
Lale, S., Azizzadenesheli, K., Hassibi, B., and Anandkumar, A. (2020a).
\newblock Logarithmic regret bound in partially observable linear dynamical systems.
\newblock {\em Advances in Neural Information Processing Systems}, 33:20876--20888.

\bibitem[Lale et~al., 2020b]{lale2020regret}
Lale, S., Azizzadenesheli, K., Hassibi, B., and Anandkumar, A. (2020b).
\newblock Regret minimization in partially observable linear quadratic control.
\newblock {\em arXiv preprint arXiv:2002.00082}.

\bibitem[Lattimore and Szepesv{\'a}ri, 2020]{lattimore2020bandit}
Lattimore, T. and Szepesv{\'a}ri, C. (2020).
\newblock {\em Bandit algorithms}.
\newblock Cambridge University Press.

\bibitem[Leqi et~al., 2021]{leqi2021rebounding}
Leqi, L., Kilinc~Karzan, F., Lipton, Z., and Montgomery, A. (2021).
\newblock Rebounding bandits for modeling satiation effects.
\newblock {\em Advances in Neural Information Processing Systems}, 34:4003--4014.

\bibitem[Levine et~al., 2017]{levine2017rotting}
Levine, N., Crammer, K., and Mannor, S. (2017).
\newblock Rotting bandits.
\newblock {\em Advances in neural information processing systems}, 30.

\bibitem[Mania et~al., 2019]{mania2019certainty}
Mania, H., Tu, S., and Recht, B. (2019).
\newblock Certainty equivalence is efficient for linear quadratic control.
\newblock {\em Advances in Neural Information Processing Systems}, 32.

\bibitem[Nakkiran et~al., 2020]{nakkiran2020optimal}
Nakkiran, P., Venkat, P., Kakade, S., and Ma, T. (2020).
\newblock Optimal regularization can mitigate double descent.
\newblock {\em arXiv preprint arXiv:2003.01897}.

\bibitem[Russac et~al., 2019]{russac2019weighted}
Russac, Y., Vernade, C., and Capp{\'e}, O. (2019).
\newblock Weighted linear bandits for non-stationary environments.
\newblock {\em Advances in Neural Information Processing Systems}, 32.

\bibitem[Sattar et~al., 2025a]{sattar2025sub}
Sattar, Y., Choi, S., Jedra, Y., Fazel, M., and Dean, S. (2025a).
\newblock Sub-optimality of the separation principle for quadratic control from bilinear observations.
\newblock {\em arXiv preprint arXiv:2504.11555}.

\bibitem[Sattar et~al., 2024]{sattar2024learning}
Sattar, Y., Jedra, Y., and Dean, S. (2024).
\newblock Learning linear dynamics from bilinear observations.
\newblock {\em arXiv e-prints}, pages arXiv--2409.

\bibitem[Sattar et~al., 2025b]{sattar2025learning}
Sattar, Y., Jedra, Y., and Dean, S. (2025b).
\newblock Learning linear dynamics from bilinear observations.
\newblock In {\em 2025 American Control Conference (ACC)}, pages 3109--3115. IEEE.

\bibitem[Sattar et~al., 2025c]{sattar2025finite}
Sattar, Y., Jedra, Y., Fazel, M., and Dean, S. (2025c).
\newblock Finite sample identification of partially observed bilinear dynamical systems.
\newblock {\em arXiv preprint arXiv:2501.07652}.

\bibitem[Schedl et~al., 2018]{schedl2018current}
Schedl, M., Zamani, H., Chen, C.-W., Deldjoo, Y., and Elahi, M. (2018).
\newblock Current challenges and visions in music recommender systems research.
\newblock {\em International Journal of Multimedia Information Retrieval}, 7(2):95--116.

\bibitem[Seznec et~al., 2019]{seznec2019rotting}
Seznec, J., Locatelli, A., Carpentier, A., Lazaric, A., and Valko, M. (2019).
\newblock Rotting bandits are no harder than stochastic ones.
\newblock In {\em The 22nd International Conference on Artificial Intelligence and Statistics}, pages 2564--2572. PMLR.

\bibitem[Simchowitz, 2020]{simchowitz2020making}
Simchowitz, M. (2020).
\newblock Making non-stochastic control (almost) as easy as stochastic.
\newblock {\em Advances in Neural Information Processing Systems}, 33:18318--18329.

\bibitem[Trella et~al., 2024]{trella2024non}
Trella, A.~L., Dempsey, W., Gazi, A.~H., Xu, Z., Doshi-Velez, F., and Murphy, S.~A. (2024).
\newblock Non-stationary latent auto-regressive bandits.
\newblock {\em arXiv preprint arXiv:2402.03110}.

\bibitem[Whittle, 1988]{whittle1988restless}
Whittle, P. (1988).
\newblock Restless bandits: Activity allocation in a changing world.
\newblock {\em Journal of applied probability}, 25(A):287--298.

\end{thebibliography}
